\pdfoutput=1
\documentclass{article}

\usepackage{iclr2027_conference}
\usepackage[T1]{fontenc}
\usepackage[utf8]{inputenc}

\usepackage{lmodern}
\usepackage{tgtermes}

\usepackage{amsmath,amssymb,amsthm}
\usepackage{booktabs,tabularx,array,multirow}
\usepackage{enumitem}
\usepackage{needspace}
\usepackage{graphicx,float,microtype}
\usepackage{xcolor}
\usepackage{listings}                %
\usepackage{tikz}
\usepackage{longtable}
\usetikzlibrary{arrows.meta,positioning,fit,backgrounds,calc,decorations.pathreplacing}
\graphicspath{{figures/}}
\usepackage[normalem]{ulem}          %

\theoremstyle{definition}
\newtheorem{definition}{Definition}
\theoremstyle{plain}
\newtheorem{proposition}{Proposition}[section]
\newtheorem{lemma}[proposition]{Lemma}
\newtheorem{theorem}[proposition]{Theorem}

\setlist[itemize]{leftmargin=1.5em,itemsep=1pt,topsep=2pt}
\setlist[enumerate]{leftmargin=1.8em,itemsep=1pt,topsep=2pt}
\newcommand{\code}[1]{\texttt{#1}}

\newif\ifshowtwins \showtwinsfalse   %
\newif\ifshowoldimp \showoldimpfalse %

\usepackage{hyperref,xurl}
\hypersetup{pdftitle={Algebraic World Model: From Exploration to Proofs},
  pdfauthor={Guancheng Pan, Qixin Tan, Yang Yuan, Andrew C. Yao},
  colorlinks=true, linkcolor=blue!60!black, citecolor=blue!60!black, urlcolor=blue!60!black}

\hypersetup{
  pdftitle={Algebraic World Model: From Exploration to Proofs},
  pdfauthor={Guancheng Pan, Qixin Tan, Yang Yuan, Andrew C. Yao}
}

\title{Algebraic World Model:\\ From Exploration to Proofs}

\author{%
  \mdseries\begin{tabular}[t]{@{}p{0.45\linewidth}p{0.45\linewidth}@{}}
    \textbf{Guancheng Pan}\thanks{These authors contributed equally.} & \textbf{Qixin Tan}\footnotemark[1] \\
    Zhejiang University & Tsinghua University \\
    \texttt{3250101086@zju.edu.cn} & \texttt{tqx23@mails.tsinghua.edu.cn} \\[10pt]
    \textbf{Yang Yuan} & \textbf{Andrew C. Yao} \\
    Tsinghua University & Tsinghua University \\
    Shanghai QiZhi Institute & Shanghai QiZhi Institute \\
    \texttt{yuanyang@tsinghua.edu.cn} & \texttt{andrewcyao@mail.tsinghua.edu.cn}
  \end{tabular}
}

\iclrfinalcopy

\begin{document}
\maketitle
\lhead{}
\lhead{Preprint}

\begin{abstract}

AI has made remarkable progress in mathematics, solving challenging problems with machine-checkable proofs in Lean, yet typically relying on predefined mathematical definitions and formal specifications. This raises a natural question: can it reason rigorously about an unknown world? We introduce the Algebraic World Model (AWM), a framework that turns interaction with unfamiliar environments into explicit mathematical models supporting formal reasoning. AWM explores the environment and develops executable Python hypotheses, continually testing and revising them against observations. It then structures these hypotheses into typed objects and shared rules, simplifies their structure through \mbox{category-theory-guided} transformations, and formalizes the resulting model in Lean~4 to prove reachability and impossibility.
AWM can formally prove that a goal is unreachable, even beyond explored trajectories. On two impossible game levels, it constructs closed, kernel-checked Lean proofs excluding every legal winning sequence in all six runs, while Claude Opus~5.5 and GPT-6 Astra, equipped with coding tools, produce none in twelve baseline runs. Four of AWM's six initial bounds are correct, compared with one for Opus~5.5 and none for GPT-6 Astra; given only feedback that a cheaper win exists, AWM revises its rules and constructs new proofs, reaching six of six correct bounds.
Moreover, AWM produces interpretable mathematical models whose rules can be inspected and tested step by step. Across all 25 public ARC-AGI-3 games, these models enable Lean to replay the winning trajectories of all 183 levels and generalize beyond observed play, predicting unseen transitions and detecting changes in game mechanics. With Claude Opus~5, AWM clears all levels with a score of 99.9, while its certified winning trajectories use only 43\% of the human baseline's actions.

\end{abstract}

\section{Introduction}
\label{sec:intro}

AI systems are advancing the frontiers of mathematics,
from Lean-verified progress on the Riemann zeta function
\citep{anthropic2026zeta} to recently reported resolutions
of the Navier--Stokes Millennium Prize problem and the
local smoothing conjecture
\citep{openai2026navierstokes,openai2026localsmoothing}. Yet these advances typically rely on predefined mathematical definitions and formal specifications. In an unfamiliar environment, however, an agent must first identify the relevant objects, infer how actions change them, and determine what constitutes success. The rules on which its reasoning depends are themselves hypotheses to be learned and revised through interaction. This raises a broader question:
\begin{center}
\emph{Can an agent reason rigorously about an unknown world?}
\end{center}

Euler's treatment of the K\"onigsberg bridges illustrates the role of mathematical structure \citep{euler1741solutio}. Land masses become vertices and bridges become edges. A walk crossing every bridge exactly once pairs arrivals and departures at every vertex other than its endpoints, so at most two vertices can have odd degree. The K\"onigsberg graph has four, ruling out every such walk without enumerating routes. A winning trajectory likewise establishes reachability, but proving that none exists requires reasoning about all legal sequences. When the environment is unfamiliar, the objects and rules needed for that reasoning must themselves be learned.

We introduce the Algebraic World Model (AWM), a framework that learns explicit mathematical models of unfamiliar environments through interaction and uses them for formal reasoning. AWM proceeds through three stages (Figure~\ref{fig:steps}). In \textbf{Explore \& Prototype}, the agent actively probes the environment and maintains executable hypotheses, testing and revising them against observed transitions. \textbf{Structure \& Simplify} organizes these hypotheses into typed objects and rules, then uses category-theory-guided transformations, each checked in Lean through 
projective and inductive
limits, to identify shared structure and simplify its representation while preserving tested distinctions. In \textbf{Formalize \& Prove}, AWM completes the model in Lean~4 \citep{demoura2021lean4} and constructs kernel-checked proofs of reachability, impossibility, and other properties. A winning witness establishes reachability; an invariant or a suitable cost lower bound can exclude every legal winning sequence under the model.

\begin{figure}[H]
\centering
\includegraphics[width=0.95\textwidth,trim=0 24bp 0 1.5bp,clip]{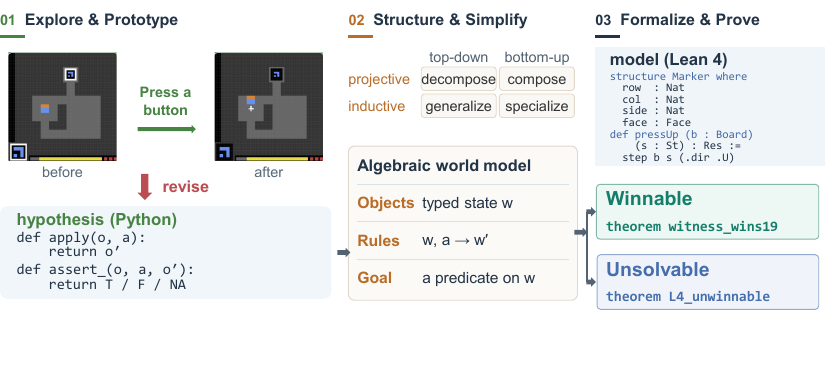}
\caption{AWM's three-stage pipeline: interactive exploration builds a provisional executable model, algebraic transformations reveal and simplify its structure, and Lean~4 formalization enables certified proofs of reachability or impossibility.}
\label{fig:steps}
\end{figure}

We test AWM's ability to prove impossibility on ARC-AGI-3 \citep{foundation2026arcagi3}, whose games reveal neither their rules nor their goals. We construct two impossible variants of \code{sc25} levels, preserving their mechanics while setting action budgets just below the minimum winning costs (\S\ref{sec:sc25}). Across three runs per level, AWM produces closed, kernel-checked Lean proofs excluding all legal winning sequences under its learned models in all six runs. Claude Opus~5.5 and GPT-6 Astra produce no such proofs across twelve baseline runs.

The difference extends beyond proof construction. AWM's initial bounds match the minimum winning costs in four of six runs, compared with one for Opus~5.5 (without a proof) and none for Astra; other baseline runs report incorrect bounds or remain inconclusive. For example, on level~4 Astra claims impossibility with an unproved minimum winning cost of 22, whereas AWM proves a tight bound of 19 under its learned model, matching the game's actual minimum winning cost. Given only feedback that a cheaper win exists, AWM revises the rules underlying its remaining two bounds and constructs new closed proofs, bringing its bound accuracy to six of six.

AWM's learned mechanisms are also explicit and inspectable: typed objects and named rules make their predictions available for step-by-step checking. Across all 25 public ARC-AGI-3 games, AWM builds compact models with a median of six object types and ten rules. These models enable Lean to replay the winning trajectories of all 183 levels, checking represented state transitions and final goal conditions (\S\ref{sec:q1}). With Opus~5, AWM clears all 25 games with a score of 99.9, competitive with leading code-executing systems; its certified winning trajectories use 43\% of the human baseline's actions.

We make three contributions.
\begin{itemize}
    \item \textbf{A framework from exploration to algebraic structure.}
    AWM turns interaction with an unknown environment into executable hypotheses and organizes them into compact models through \mbox{category-theory-guided} transformations checked in Lean, presenting object types as projective and inductive limits. Shared, parameterized rules express related mechanisms in a common form, allowing their properties to be stated and proved uniformly (\S\ref{sec:method}).

    \item \textbf{Formal impossibility reasoning in unknown worlds.}
    AWM constructs closed, kernel-checked Lean proofs that rule out all legal winning sequences under its learned rules and specified constraints, including sequences never explored. This establishes model-relative impossibility rather than inferring it from unsuccessful play (\S\ref{sec:sc25}).

    \item \textbf{Inspectable models and traceable proofs.}
    Typed objects and named rules make learned dynamics understandable and their predictions checkable against observations. Lean replay connects the formal representation to recorded behavior, while explicit proof dependencies identify the rules on which a conclusion rests, supporting targeted empirical examination and revision (\S\ref{sec:q1}; \S\ref{sec:repair}).
\end{itemize}
\section{Method}
\label{sec:method}\label{sec:awm}\label{sec:construction}

AWM aims to turn interaction with an unknown environment into a mathematical model that supports rigorous reasoning about both reachability and impossibility. This requires more than collecting rules that explain observed transitions: the rules must form a coherent representation of the environment, and that representation must support reasoning about action sequences the agent has never explored.

We first describe the target model (\S\ref{sec:modelstructure}), then present the three stages that construct and use it (Figure~\ref{fig:steps}): Explore \& Prototype develops executable hypotheses through interaction; Structure \& Simplify organizes them into a compact algebraic model; and Formalize \& Prove completes the model in Lean~4 and establishes machine-checked results.

\subsection{The Algebraic World Model}
\label{sec:modelstructure}

An algebraic world model describes an environment through typed objects, rules, and a goal predicate, represented in executable Python and formal Lean~4 to support both empirical testing and mathematical reasoning.

\paragraph{Objects.}
Object types define the structure and validity of states, providing a common vocabulary for describing the environment. A parser maps observations, together with any required history or auxiliary information, to typed states $w$.

\paragraph{Rules.}
Rules describe how actions transform these states. Each Python rule specifies an applicability scope and a checker that reports agreement, contradiction, or non-applicability on observed transitions; effect rules also predict successor states or their components. Together, the rules define legal actions, their effects and costs $\mathrm{cost}(w,a)$, while the goal predicate identifies successful states.

Related mechanisms can share a parameterized operation $r(\theta)$, with $\theta$ determined by the state. This allows common properties to be expressed and proved uniformly across different mechanisms.

\paragraph{Formal representation.}
The same model is formalized in Lean~4, with objects and rules presented as projective and inductive limits whose universal properties the kernel checks, and properties of their dynamics expressed as theorems. Its executable representation supports testing against observations, while its formal representation supports kernel-checked proofs over legal action sequences (\S\ref{sec:rules2axioms}).

\subsection{Explore \& Prototype}
\label{sec:propose}

\paragraph{Setting.}
We consider deterministic, goal-directed environments with partial observations, modeled in the POMDP setting \citep{kaelbling1998pomdp}. At step $t$, the agent observes $o_t$, takes an action $a_t$, and receives $o_{t+1}$ without access to the underlying state or transition function. Its experience forms a corpus $\mathcal{D}_t$ of distinct transitions $x=(o,a,o')$.

\paragraph{Building and testing hypotheses.}
AWM actively explores the environment and maintains a provisional executable model as a \emph{concept book}. Each concept is a Python hypothesis about an environmental mechanism, containing a natural-language claim, an applicability scope, supporting transitions, a checker, and an optional predictor.

For a concept $c$, its checker returns $v_c(x)$: \code{TRUE} if the transition agrees with the hypothesis, \code{FALSE} if it contradicts it, and \code{NA} if it falls outside its scope. The optional predictor specifies expected action effects for comparison with subsequent observations.

After each action batch, AWM rechecks the concept book $\mathcal{C}_t$ against the expanded corpus $\mathcal{D}_t$. Further interaction is allowed only when
\begin{equation}\label{eq:gate}
\forall c\in\mathcal{C}_t,\ \forall x\in\mathcal{D}_t:
\quad v_c(x)\notin\{\mathsf{FALSE},\mathsf{ERR}\},
\end{equation}
where $\mathsf{ERR}$ denotes checker failure, including exceptions and timeouts.

This recheck ensures that contradicted or failing concepts must be revised or withdrawn before exploration continues. The agent then keeps revising its hypotheses and uses them to plan toward inferred goals or probe competing explanations.

The resulting prototype captures tested mechanisms but may describe related behaviors independently, motivating the next stage.

\subsection{Structure \& Simplify}
\label{sec:consolidate}

\begin{figure}[t]
\centering
\includegraphics[width=\textwidth]{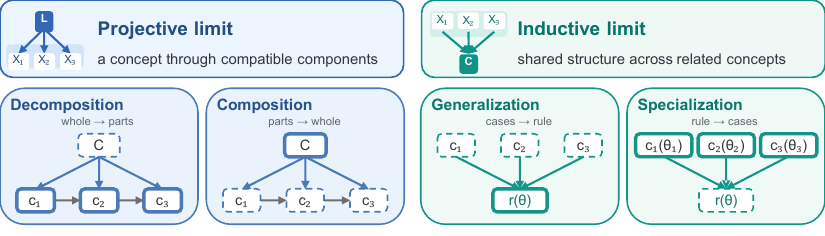}
\caption{Four transformations guided by projective and inductive limits, applied to object types and rules. Dashed boxes mark removed components and bold boxes mark introduced components; arrows indicate relationships between components. Generalization uses $r(\theta)$, and an instance is written $r(\theta_i)$; specialization makes distinctions explicit through case-specific types or parameters.}
\label{fig:condensation}
\end{figure}

Structure \& Simplify turns the concept book's fragmented, overlapping descriptions into a compact model by identifying shared structure and reorganizing components through their relationships. AWM first extracts the required object types, constructs a parser from observations to typed states, and rewrites concepts as rules over these objects. It then restructures these \emph{components}---object types and rules---through four \emph{category-theory-guided} transformations. A \emph{projective limit} (limit) assembles compatible descriptions into a coherent whole, while an \emph{inductive limit} (colimit) glues objects along shared structure \citep{maclane1998categories,kashiwara2006categories}. \citet{yuan2023succinct} proposes representing a concept as such a limit of simpler concepts and verifying the representation at each node; AWM is a concrete implementation of this idea for learned world models, in which the two perspectives guide transformations based on compatibility and shared structure (Figure~\ref{fig:condensation}), each with an auxiliary Lean statement about cones, cocones and universal properties checked by the kernel (Appendix~\ref{app:limits}).

Following projective limits, \emph{decomposition} expresses a complex component through simpler components and their compatibility relations; Lean checks that it is the limit of the resulting diagram. Conversely, \emph{composition} expresses a component as the limit of compatible existing components, which Lean checks, moving the composite into the derived layer when a separate primitive definition is unnecessary.

Following inductive limits, \emph{generalization} extracts shared structure from multiple components, expressing source rules $c_i$ through a parameterized operation $r(\theta_i)$ with state-dependent $\theta_i$; Lean checks that $r$ is the copairing $[c_i]$ out of the coproduct of cases. Conversely, \emph{specialization} refines overly general components into more specific cases or introduces parameters to reduce exceptions and redundancy; Lean checks that the cases form a coproduct.

\paragraph{Preserving tested distinctions.}
Every transformation must compile its auxiliary Lean statement and preserve verdicts on recorded transitions. For rules, AWM also uses a diagnostic set $\mathcal{N}$ of perturbed recorded transitions to check whether simplification loses previously detected rejections.

Define $\mathrm{Rej}(c)=\{x\in\mathcal{N}:v_c(x)=\mathsf{FALSE}\}$ and $\mathrm{Rej}(S)=\bigcup_{c\in S}\mathrm{Rej}(c)$. Replacing rules $S$ with $T$ must preserve coverage of recorded transitions and satisfy
\begin{equation}\label{eq:unify}
\mathrm{Rej}(T)\supseteq\mathrm{Rej}(S).
\end{equation}
For object types, which have no checkers, transformations are evaluated by distinctions between recorded states, using state pairs as witnesses, and by rule verdicts once rules are available.

AWM applies these transformations and tests to object types, then rules, and finally rechecks object types against the settled rules.
\label{sec:algebraicstructure}

\subsection{Formalize \& Prove}
\label{sec:rules2axioms}

AWM completes the model in Lean~4 from the simplified objects and rules. Both are presented by projective and inductive limits (\S\ref{sec:modelstructure}); in \code{lp85}, a world is a pullback, the goal an equalizer, and a rule defined by cases a copairing out of a coproduct (Appendix~\ref{app:limits}). Rules are \code{def} definitions and derived properties \code{theorem} statements, with auxiliary proofs from Structure \& Simplify kept separate. The transition function dispatches to named rules, while unresolved assumptions remain explicit theorem hypotheses rather than unchecked \code{axiom} declarations. Replaying recorded winning trajectories, the Lean model reproduces observed behavior on every component it represents, grounding its formal definitions in the dynamics the agent encountered. This enables theorems about all legal action sequences under the learned model, extending formal reasoning beyond the trajectories actually explored.

\paragraph{Proving reachability and impossibility.}
\label{sec:modelreasoning}
A level is \emph{winnable} in the model if a legal action sequence reaches the goal. AWM proves reachability by constructing a witness checked in Lean. Impossibility can arise because the goal is unreachable, or because reaching it requires more cost than the available budget. AWM proves the former using an invariant that excludes the goal, and the latter using a lower bound on the cost of any winning sequence. For budgeted tasks, a potential $V$ estimates the minimum remaining cost from each state and satisfies
\begin{equation}\label{eq:potential}
V(w)\le\mathrm{cost}(w,a)+V(w')
\quad\text{for every legal transition},
\qquad V(w)\le0\quad\text{on the goal},
\end{equation}
so every winning sequence costs at least $V(w_0)$, proving impossibility when this exceeds the budget. All proofs are closed and checked by the Lean kernel, without unproved placeholders (\code{sorry}) or kernel-bypassing computation (\code{native\_decide}), establishing rigorous guarantees under the learned model's explicit premises. More generally, we have the following definition for proving impossibilities. 

\begin{definition}[Sufficient cover]\label{def:coverage}
Given a specific task $T$, 
a learned model is a \emph{sufficient cover} of $T$,  if every optimal winning sequence in the environment is covered in the model, i.e.,
the transitions in the winning sequence can be  reproduced step by step in the model. 
\end{definition}

Under sufficient cover, impossibility proofs in the model transfer to the environment (Appendix~\ref{app:lean}, Theorem~A.1).

\subsection{Example: one rule for four mechanisms}
\label{sec:case}

In game \code{lp85}, buttons move colored tiles, and a level is won when every slot holds its matching color. During Explore \& Prototype, the agent discovers eleven concepts describing four button mechanisms: belts, stars, hubs, and panels (Figure~\ref{fig:algebra}a). In Structure \& Simplify, AWM applies the generalization transformation to recognize that all four mechanisms permute tile positions, replacing the eleven concepts with a shared operation, \code{press\_permutes}, parameterized by the pressed button. On the diagnostic set used for this merge, the shared rule retains all 55 rejections of the source concepts and rejects 21 additional corrupted transitions. The shared rule also shortens reasoning: because one color-count invariant covers every button, proofs that a variant is unsolvable need two to three times fewer lemmas than proofs over the eleven separate concepts (Appendix~\ref{app:e3}). A separate check tests whether such a model-level proof also holds for the game: on the level-6 tile-swap variant U, Lean verifies its unsolvability proof against an engine-derived graph covering all 52,500 transitions among 7,500 configurations (Appendix~\ref{app:engine-proof}).

In Formalize \& Prove, Lean certifies a five-click winning sequence (Figure~\ref{fig:algebra}b) and proves a recolored variant unwinnable using the color-count invariant shared by all button permutations (Figure~\ref{fig:algebra}c). The final model contains six object types and six rules, and Lean replays all 199 recorded transitions without mismatch.

\begin{figure}[t]
\centering
\includegraphics[width=0.92\textwidth]{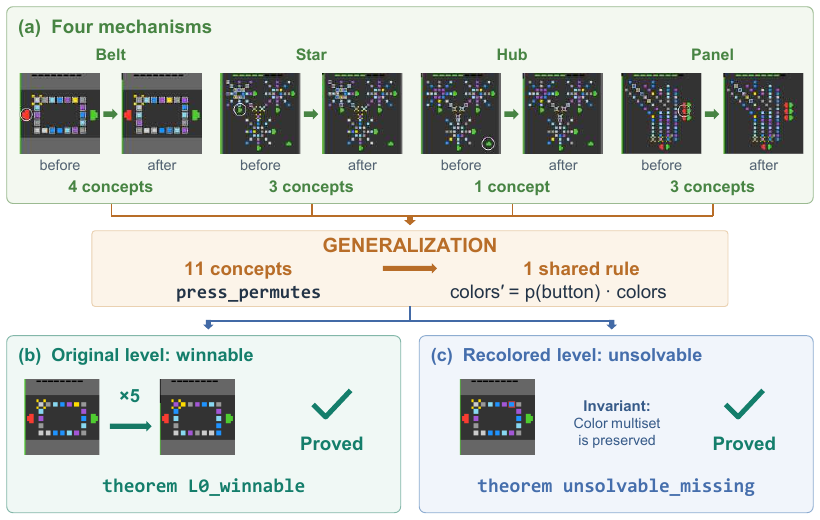}
\caption{
From shared structure to formal proofs in \code{lp85}.
(a) Eleven concepts describing four button mechanisms are generalized into a shared permutation rule. Each before/after pair shows a recorded press, with the pressed button circled and changed tiles outlined.
(b) A five-click sequence witnesses reachability.
(c) Recoloring the only yellow tile makes the goal unreachable, since every press preserves color counts.
Both conclusions are formally proved in Lean~4.
}
\label{fig:algebra}\label{fig:lp85}
\end{figure}

\section{Twenty-five public games}
\label{sec:q1}

ARC-AGI-3 \citep{foundation2026arcagi3} scores each level by the squared ratio of a human baseline's action count to the agent's, weights later levels more, and gives an uncleared level zero. We run AWM with Claude Opus~5 on all 25 public games (183 levels), starting each game with an empty concept book. AWM clears every level with a score of 99.9 and the fewest certified actions among the systems reporting this metric in Table~\ref{tab:scores}. We further examine whether the resulting models are compact, reproduce observed behavior in Lean, and generalize beyond the trajectories from which they were learned.

\begin{table}[t]
\centering\footnotesize
\setlength{\tabcolsep}{4pt}
\caption{Official scores on the 25 public ARC-AGI-3 games. Claude Opus~5 and GPT-6 Astra run under the official standard harness, which provides no code execution; the other systems execute code. AWM's certified clears follow uncertified sessions of the same run in 22 games. The actions for Rodionov's agent are from its ARC Prize scorecard and Memento~3's as reported in~\citep{zhao2026memento3}.}
\label{tab:scores}
\begin{tabular}{@{}lcrrrr@{}}
\toprule
System & Code & \begin{tabular}[b]{@{}c@{}}Mean\\score\end{tabular} & \begin{tabular}[b]{@{}c@{}}Games\\at 100\end{tabular} & \begin{tabular}[b]{@{}c@{}}Certified\\actions\end{tabular} & \begin{tabular}[b]{@{}c@{}}Actions\\/ human\end{tabular} \\
\midrule
Claude Opus~5 & no & 40.7 & 5 & --- & --- \\
GPT-6 Astra & no & 68.3 & 10 & --- & --- \\
GPT-6 Astra (adapter) \citep{arcprize2026astraresults} & yes & 100.0 & 25 & --- & --- \\
\citet{rodionov2026need} & yes & 99.0 & 23 & 8{,}347 & 0.49 \\
Memento~3 \citep{zhao2026memento3} & yes & 100.0 & 25 & 7{,}518 & 0.44 \\
AWM (ours) & yes & 99.9 & 23 & \textbf{7{,}448} & \textbf{0.43} \\
\bottomrule
\end{tabular}

\medskip
\caption{Formalization and verification of the 25 learned world models in Lean~4. Axioms are those Lean's own logic uses; no proof adds axioms of its own.}
\label{tab:lean}
\begin{tabular}{@{}cccc@{}}
\toprule
Levels replayed & Rules formalized in Lean & Theorems & Axioms \\
\midrule
183 / 183 & 220 & 1{,}374 & only \code{propext}, \code{Classical.choice}, \code{Quot.sound} \\
\bottomrule
\end{tabular}
\end{table}

\paragraph{Compact models that reproduce observed behavior.}
Across the 25 games, AWM records 742 executable concepts and withdraws 143 during model revision. The retained concepts are tested against the recorded transitions. Structure \& Simplify organizes the rest into 170 object types and 283 rules, a median of 6 types and 10 rules per game. Shared patterns include translations, permutations and goal conditions. Lean replays the winning trajectory of every one of the 183 levels from its initial state and verifies its final state against the modeled goal (Table~\ref{tab:lean}).

\paragraph{Generalization beyond observed transitions.}
We next test the learned rules on behavior the record never showed. We freeze the models of \code{lp85} and \code{ft09} after exploration and evaluate them on transitions from random play with previously unseen screens (seeds 0--9; Appendix~\ref{app:unseen}). The final frozen rules accept every valid transition in the new-trajectory sets. On the two modified \code{lp85} layout sets, they predict the next tile layout in all 4,160 cases; on the two changed-mechanic \code{ft09} sets, they reject all 4,031 transitions. On a third game, \code{ls20}, the same evaluation accepts 99.47--99.95\% of valid transitions on previously unseen screens with the final rules, compared with 67--70\% with the concepts before Structure \& Simplify (Appendix~\ref{app:unseen}).

The changed mechanic modifies \code{ft09}'s hidden click pattern, causing each click to change 110 or 146 cells instead of 38; rejecting these transitions shows that the rules detect behavior inconsistent with the learned dynamics.

\paragraph{Clearing the games with efficient solutions.}
The same models are effective for play. AWM with Opus~5 achieves a score of 99.9, compared with 40.7 for Opus~5 under the official standard harness without code execution. It scores 100 on 23 games and at least 97.6 on every game, performing competitively with leading code-executing systems (Table~\ref{tab:scores}). Its certified clears require 7{,}448 actions, only 43\% of the human baseline's 17{,}135 and the fewest among systems reporting this metric, compared with 7{,}518 for Memento~3 \citep{zhao2026memento3} and 8{,}347 for \citet{rodionov2026need}. Including exploration and other sessions, AWM uses 24{,}604 engine actions in total (Appendix~\ref{app:scores}). These results show that AWM can construct and use explicit world models while retaining competitive solution efficiency. \S\ref{sec:sc25} next examines whether such models also support formal proofs of impossibility.

\section{Proving impossibility}
\label{sec:sc25}

A winning trajectory establishes reachability; impossibility requires excluding every legal winning sequence within the budget, including those never explored. We test whether AWM's learned algebraic models enable such exhaustive proofs, and whether those proofs hold for the game itself.

\subsection{The game, the seal, and the baselines}
\label{sec:gt}\label{sec:sc25arms}

In \code{sc25}, a block moves on a grid and casts spells through a $3\times3$ panel to reach the exit within the level's budget. Moves and panel clicks cost one unit, completing a spell one more, fuel refunds ten, and the winning move is free. We construct two variants of levels~4 and~5 by rearranging selected elements while preserving the game mechanics, then seal them at budgets of 18 and~35. Exhaustive search on the game engine (the underlying game program), covering all four initial facings and all 13 actions, establishes minimum winning costs of 19 and~36, respectively: neither sealed level can be won, while one additional unit permits a win.

Claude Opus~5.5 and GPT-6 Astra receive the game interface and coding tools, including simulator construction and search; AWM uses Opus~5.5 with the procedure of \S\ref{sec:method} and a generic Lean skeleton. Each condition has three runs per level, with up to eight reset-delimited attempts per run, ending at game over or recorded abandonment. After each attempt, every agent is asked whether the sealed level can be won and to justify its verdict; AWM constructs Lean proofs under its learned model. Impossibility can be established by proving that the minimum winning cost exceeds the budget. We evaluate proof validity under the learned model separately from numerical agreement with the engine.

\subsection{Certified lower bounds}
\label{sec:results}\label{sec:sc25l4}\label{sec:sc25l5}\label{sec:sc25proof}

\begin{table}[t]
\centering\footnotesize
\caption{Impossibility reasoning across six runs per agent. Bound accuracy is evaluated against the games' minimum winning costs before and after feedback-guided repair. First claim averages the earliest impossibility claim among runs making one, including later withdrawals; it is not proof\mbox{-}completion time. Per-run results are in Appendix~\ref{app:mainruns}.}
\label{tab:verdicts}\label{tab:l5verdicts}
\setlength{\tabcolsep}{4.5pt}
\begin{tabular}{@{}lcccc@{}}
\toprule
Agent &
\begin{tabular}[b]{@{}c@{}}Kernel-checked\\model proofs\end{tabular} &
\begin{tabular}[b]{@{}c@{}}Runs claiming\\impossible\end{tabular} &
\begin{tabular}[b]{@{}c@{}}Correct bounds\\initial $\to$ feedback\end{tabular} &
\begin{tabular}[b]{@{}c@{}}First claim\\(attempt)\end{tabular} \\
\midrule
Claude Opus~5.5 & 0/6 & 3/6 & 1/6 $\to$ 1/6 & 5.3 \\
GPT-6 Astra & 0/6 & 2/6 & 0/6 $\to$ 2/6 & 8.0 \\
\textbf{AWM} & \textbf{6/6} & \textbf{6/6} & \textbf{4/6 $\to$ 6/6} & \textbf{2.0} \\
\bottomrule
\end{tabular}
\end{table}

\paragraph{Proving beyond explored trajectories.}
All six AWM runs produce closed, kernel-checked Lean proofs of impossibility under their learned models and explicit premises, while none of the twelve baseline runs produces such a proof (Table~\ref{tab:verdicts}). Crucially, these proofs exclude all legal winning sequences within the budget under the learned models, including those never explored. For example, an AWM run for level~4 proves
\begin{equation}\label{eq:sealed-certificate}
\forall B\le18,\ \forall\mathbf{a},\quad
\mathrm{Won}(\mathrm{run}_{B}(w_0,\mathbf{a}))=\mathsf{false},
\end{equation}
and verifies a winning witness at $B=19$. Together, the lower bound and witness certify a tight threshold under the model, whereas finding a winning route alone cannot exclude cheaper alternatives.

AWM establishes such bounds through potential or closure certificates over arbitrary legal action lists. Certificate data may come from any search; Lean checks that the certificate proves the stated bound under the formal model, so the proof need not trust the search implementation. We conduct additional tests showing that our Lean-based verification procedure can detect errors such as reversed bounds, Python--Lean mismatches, and vacuous theorem statements (Appendix~\ref{app:lean-python}).

\paragraph{Beyond unproved numerical judgments.}
Only five of twelve baseline runs claim impossibility, four with incorrect bounds. Even the sole correct Claude Opus~5.5 bound lists ``that my hand-counted path lengths are truly minimal'' as \emph{assumed, not checked}. Finding a winning route establishes an upper bound on the minimum cost, not its optimality. AWM instead provides kernel-checked lower-bound proofs under its learned models.

Figure~\ref{fig:l4attempts} illustrates this distinction on L4. Astra finds a winning route costing 22 and assumes it is optimal without proof, incorrectly ruling out wins at budgets 19--21. In contrast, AWM formally proves a lower bound of 19 under its learned model and verifies a matching winning witness, establishing a tight bound that agrees with the game's minimum winning cost.

\paragraph{Proof validity and bound accuracy are distinct.}
Before feedback, four of six AWM bounds are correct, compared with one for Opus~5.5 and none for GPT-6 Astra. 

Some runs identify a correct winning sequence but charge the otherwise free winning move, reporting a bound one unit higher. Because this cost cannot be observed on sealed levels, we count such bounds as correct for all agents only when the discrepancy is explicitly attributable to this accounting convention, not to an incorrect route or cost estimate.

The remaining two AWM bounds, 20 on L4 and 47 on L5, arise from overly restrictive learned rules that omit transitions used by the game's cheapest winning routes, violating sufficient cover (\S\ref{sec:rules2axioms}). Their proofs remain valid under those models, although their numerical thresholds disagree with the game.

\begin{figure}[t]
\centering
\includegraphics[width=\textwidth]{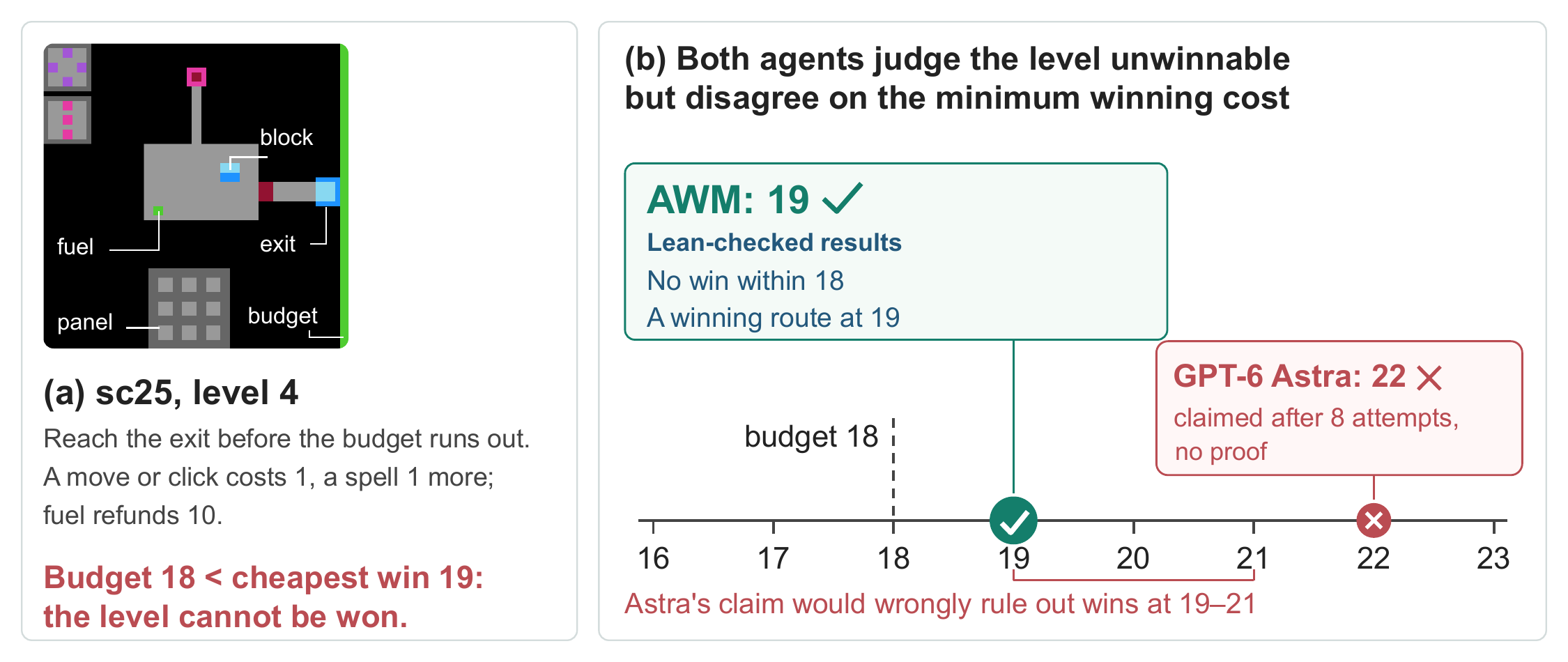}
\caption{Sealed sc25 level 4. AWM proves in Lean that the minimum winning cost under its learned model is 19, one unit above the budget, and verifies a route at that cost; GPT-6 Astra claims 22 without a proof.}
\label{fig:l4attempts}
\end{figure}

\subsection{Auditing premises and proving again}
\label{sec:repair}\label{sec:validproof}

After the initial evaluation, the two AWM runs with bounds of 20 and 47 receive feedback only that the game permits a cheaper win, without revealing its cost, route or the faulty rule. AWM revisits its explicit rules, tests alternative formulations through interaction, and constructs new proofs under the refined models (Appendix~\ref{app:repair}).

On L4, AWM discovers that growing onto fuel collects it, reducing the bound from 20 to~19. On L5, it identifies terrain rather than limited range as the cause of a failed shot, reducing the bound from 47 to 37 under its convention of charging the winning move, equivalent to the game's minimum cost of 36. Both refinements succeed within four additional attempts, each yielding a new closed Lean proof. With this feedback, AWM reaches correct bounds in all six runs.

We provide the same discrepancy feedback and additional attempt budget to the four baseline runs with incorrect bounds. Astra corrects both bounds but produces no formal proofs, while Claude Opus~5.5 corrects neither.

\paragraph{Proof dependencies help locate faulty rules.}
Both repairs corrected rules used by the original proofs. To test whether proof dependencies help locate such errors, we inject a single incorrect rule into a reference model and give Claude Opus~5.5 the faulty model, its Lean proof, and access to the game. The goal is to find a transition where the model's prediction disagrees with the game. On L4, Claude succeeds in 16 of 42 runs without explicit proof dependencies. Providing the rules used by the proof, without revealing which rule is faulty, increases this to 29 of 42 runs. Even when directly told which rule is faulty, Claude finds a counterexample in only 30 of 42 runs. Thus, proof dependencies substantially improve error discovery, nearly matching the benefit of identifying the faulty rule directly (Appendix~\ref{app:e1}).

\subsection{From model proofs to the engine}
\label{sec:enginetransfer}

Matching the game's minimum winning cost provides evidence of accurate reasoning, but does not establish that a formal proof under the learned model also holds for the game. To transfer such a proof, the learned model need not cover every game transition; by Definition~\ref{def:coverage}, it suffices that it reproduce the game's optimal winning sequences. For the budgeted tasks here, we verify sufficient certificate-level conditions on engine-derived graphs, which preserve budget feasibility and the proof invariants (Lemmas~\ref{lem:budget-transfer}--\ref{lem:invariant-transfer}).

We exhaustively check the engines’ budget-feasible transitions from the fixed opening against the potential inequalities and invariant conditions used in the model proofs. These checks establish that five of the six model-level impossibility proofs can be extended to the original games. Appendix~\ref{app:transfer} gives the per-run transfer arguments, including the completed rules and certificates.

\section{Related work}
\label{sec:related}

\paragraph{Executable world models.}
Code world models are learned from interaction for planning \citep{tang2024worldcoder}, refined by generation, testing and repair \citep{dainese2024cwm} or composed from programmatic experts \citep{piriyakulkij2025poeworld}; theory-based agents explore and plan with learned theories \citep{tsividis2021theorybased}, and TheoryCoder-2 learns reusable abstractions for hierarchical planning \citep{ahmed2026abstractions}. Such models are central to recent ARC-AGI-3 agents \citep{rodionov2026executable,courtis2026opineworld,rodionov2026need}: Twin does not act until its simulator reproduces every observed transition \citep{skoutnev2026twin}, and Tycho finds that a closer match does not by itself yield better play \citep{lehmann2026tycho}. The closest work, Memento~3, compiles a natural-language rulebook into code and accepts an update only if an LLM judges it faithful and replay is exact \citep{zhao2026memento3}. Yet a code world model can pass a sampled-transition check at 100\% accuracy and still lose because it misses a key rule \citep{aguilar2026verified}. AWM instead abstracts its rules into a shared structure, tests them beyond the record, and proves reachability and impossibility, checking the proofs' premises against the game.

\paragraph{Certified impossibility.}
Unsolvability certificates let an independent checker confirm that a planning task has no plan \citep{eriksson2017unsolvability}, in the spirit of certifying algorithms \citep{mcconnell2011certifying}, and certificate checking has been formally verified for temporal planning \citep{wang2025verified}; all start from a supplied model. Impossibility transfers from an abstraction that covers every environment transition, with preserved initialization and goals \citep{clarke1994abstraction}; Definition~\ref{def:coverage} requires only that the model reproduce the optimal winning sequences; action-model learning gives conditions under which learned models support such transfer \citep{aineto2024guarantees}. AWM learns its model from interaction; Lean checks every deduction from that model, and interaction checks the model's rules against the game.

\paragraph{Formal reasoning.}
Neural provers produce Lean proofs at olympiad level \citep{achim2025aristotle,chen2025seedprover,hubert2025olympiad,ren2025deepseekproverv2}, AI systems have produced Lean proofs of open Erd\H{o}s problems \citep{sothanaphan2026erdos728,tsoukalas2026formalsearch}, and autoformalization translates informal statements into formal ones \citep{wu2022autoformalization}. There the objects of proof are given mathematics; AWM must first infer the dynamics and costs its proofs quantify over, which makes the fidelity of its formal model an empirical question.
Large-scale AI-assisted software verification proves properties of compiled programs relative to machine semantics and theorem statements, and can expose bugs during proof construction \citep{lucas2026formalverification}. That work also studies Claude-generated x86 semantics; AWM learns the objects, action rules and goals of an unfamiliar interactive environment.

\section{Limitations and conclusion}
\label{sec:conclusion}

\paragraph{Limitations.}
We demonstrate AWM on ARC-like environments with discrete actions and immediate, observable feedback, where learned rules can be checked against recorded transitions; richer hidden state, delayed feedback and stochastic dynamics remain to be tested. Lean guarantees the correctness of proofs under the learned model's explicit premises, and whether these premises capture the environment is tested by interaction: the explicit rules and kernel-checked proofs separate errors in the learned model from errors in reasoning, so a discrepancy can be traced to its assumption and addressed by further exploration. Broader environments and tasks are the natural next test of such formal reasoning.

\paragraph{Conclusion.}
Formal reasoning traditionally begins with predefined objects and rules. AWM takes a step toward reasoning rigorously about unknown worlds by discovering these structures through interaction, organizing them into compact algebraic models, and formalizing them in Lean. The resulting models make learned mechanisms explicit and their behavior checkable step by step, while supporting kernel-checked proofs of reachability and impossibility, including claims about action sequences never explored. By making the assumptions behind these proofs explicit, AWM also allows discrepancies with the environment to be traced to learned rules and addressed through further exploration. Together, these capabilities bridge empirical exploration and formal reasoning, extending mathematical proof beyond problems whose definitions are given in advance.

\subsection*{Reproducibility Statement}

Our interactive website presents the models, tests, Lean proofs,
and links to official replays for all 25 public games:
\url{https://ephemeral6.github.io/awm-site/}.
The Lean model of \code{lp85}, our running example, is also
included in the supplementary materials (\code{anc/lp85-lean}),
together with scripts for building the project and auditing
its axioms.
The website also provides the results and supporting materials
for the sealed-level experiments, with per-run results reported
in Appendix~\ref{app:mainruns}.

\section*{Statement on AI use}
The agents studied in this paper are LLM agents, so AI use is part of the method. Generative
AI was also used to draft and edit the text and to write the analysis scripts. All
reported numbers come from script outputs over the recorded runs, not from
model-generated text. The Lean proofs are checked by the Lean kernel, independently of
any model.

\bibliographystyle{iclr2027_conference}
\bibliography{Reference}

\clearpage
\appendix
\newcounter{apppar}[section]
\renewcommand{\theapppar}{\thesection.\arabic{apppar}}
\newcommand{\apppar}[2]{\refstepcounter{apppar}\paragraph{\theapppar\ #1}\label{#2}}
\section{Method details}
\label{app:method}
This appendix gives the definitions behind \S\ref{sec:method}: concepts and the gate (\S\ref{sec:propose}), the simplification test (\S\ref{sec:consolidate}), and the Lean statements, certificates and their transfer to the environment (\S\ref{sec:rules2axioms}).

\apppar{Environment, concepts and the gate.}{app:gate}
We model an environment as a deterministic, goal-directed special case of a partially observable Markov decision process \citep{kaelbling1998pomdp}: states $\mathcal{S}$, actions $\mathcal{A}$, observations $\mathcal{O}$, a transition function $T:\mathcal{S}\times\mathcal{A}\to\mathcal{S}$, an observation function $\rho:\mathcal{S}\to\mathcal{O}$ and goal states $G\subseteq\mathcal{S}$. The agent sees $o_t=\rho(s_t)$ and receives a signal when a level is cleared. A transition is a triple $x=(o,a,o')$, and a game's transition corpus $\mathcal{D}$ is the set of distinct transitions across its runs, deduplicated by content hash.

A concept $c=(\sigma_c,\alpha_c,\gamma_c,E_c)$ has an optional executable scope $\sigma_c$, an assertion $\alpha_c$ returning \textsc{true}, \textsc{false} or \textsc{na}, an optional predictor $\gamma_c$, and a non-empty evidence set $E_c$. When the scope rejects $x$ the verdict $v_c(x)$ is $\mathsf{NA}$, otherwise it is $\alpha_c(x)$; an exception, an invalid return value or a timeout gives $\mathsf{ERR}$. A concept is scored only on transitions outside $E_c$. The gate of Eq.~\eqref{eq:gate} is enforced on the action command: it refuses the next action while the latest settlement has a $\mathsf{FALSE}$ or $\mathsf{ERR}$ verdict on any $x\in\mathcal{D}$, including the concept's own evidence. Figure~\ref{fig:concept} shows one concept.

\begin{figure}[!htbp]
\centering
\includegraphics[width=\linewidth]{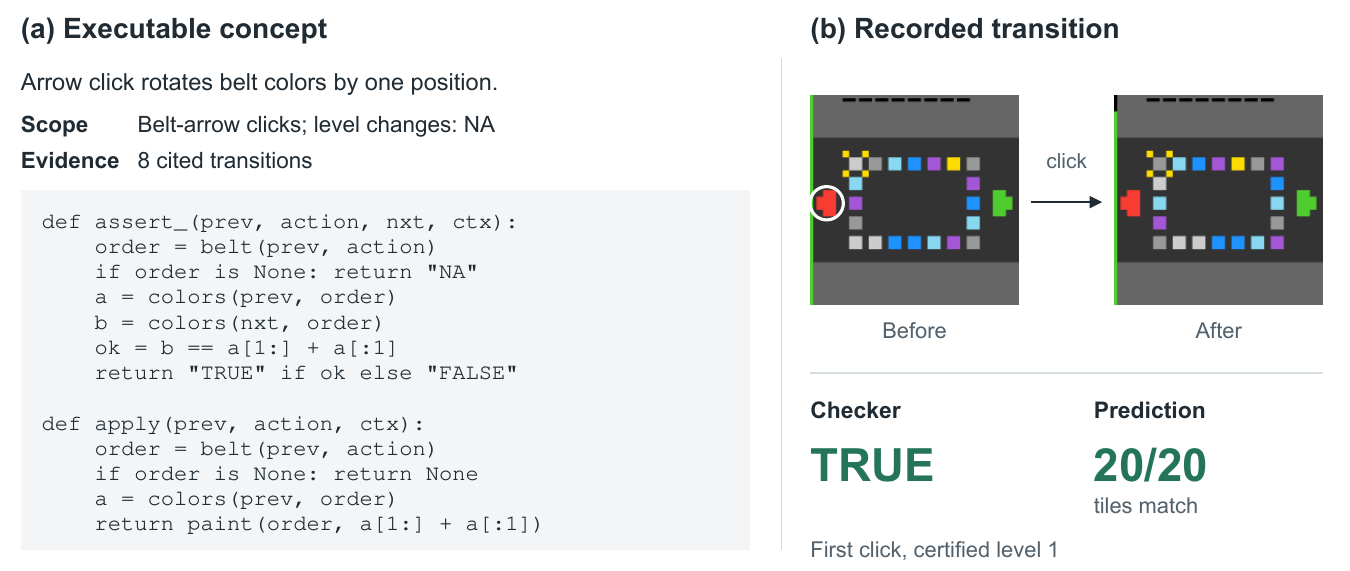}
\caption{An executable concept in \code{lp85}. (a)~Scope, evidence and abridged code for a click on an arrow pointing against its belt's order; helper names are simplified. (b)~On the first click of the certified level-1 run, the checker returns \code{TRUE} and the prediction matches all 20 tiles.}
\label{fig:concept}
\end{figure}

\apppar{Objects, rules and the simplification test.}{app:simplify}
A partial parser $\pi:\widetilde{\mathcal{O}}\rightharpoonup\mathcal{W}$ maps an observation, augmented by any explicitly supplied history or auxiliary information, to a world of typed objects; $\pi(o)$ abbreviates parsing with that context. Pixel interpretation is confined to the parser, and parse failures are reported rather than guessed. A rule $r=(\sigma_r,\alpha_r)$ has a scope and an assertion on $(\pi(o),a,\pi(o'))$. The diagnostic set $\mathcal{N}$ of Eq.~\eqref{eq:unify} contains at most 120 corrupted transitions per game, with a fixed seed. Each replaces part of a recorded transition $(o,a,o')$ in one of five ways: the successor is $o$ itself; the successor of the same action on another transition; the predecessor of another transition; $o'$ with a few cells changed that do vary in the corpus; or another transition's $o$. We set
\begin{equation}\label{eq:A-rej}
\mathrm{Rej}(r)=\{m\in\mathcal{N}: v_r(m)=\mathsf{FALSE}\}. \tag{A.1}
\end{equation}
The following is an auxiliary diagnostic within the four transformations, not a separate operation or a sufficient deletion criterion: rule $r$ has no unique rejection on $\mathcal{N}$ when
\begin{equation}\label{eq:A-elim}
\mathrm{Rej}(r)\subseteq\textstyle\bigcup_{r'\neq r}\mathrm{Rej}(r'). \tag{A.2}
\end{equation}
Any removal must preserve model behavior and pass the same before/after checks; an analogous field diagnostic fixes a field to a constant and compares verdicts on $\mathcal{D}\cup\mathcal{N}$. Figure~\ref{fig:merge} shows the test on \code{lp85}.

\begin{figure}[!htbp]
\centering
\includegraphics[width=\linewidth]{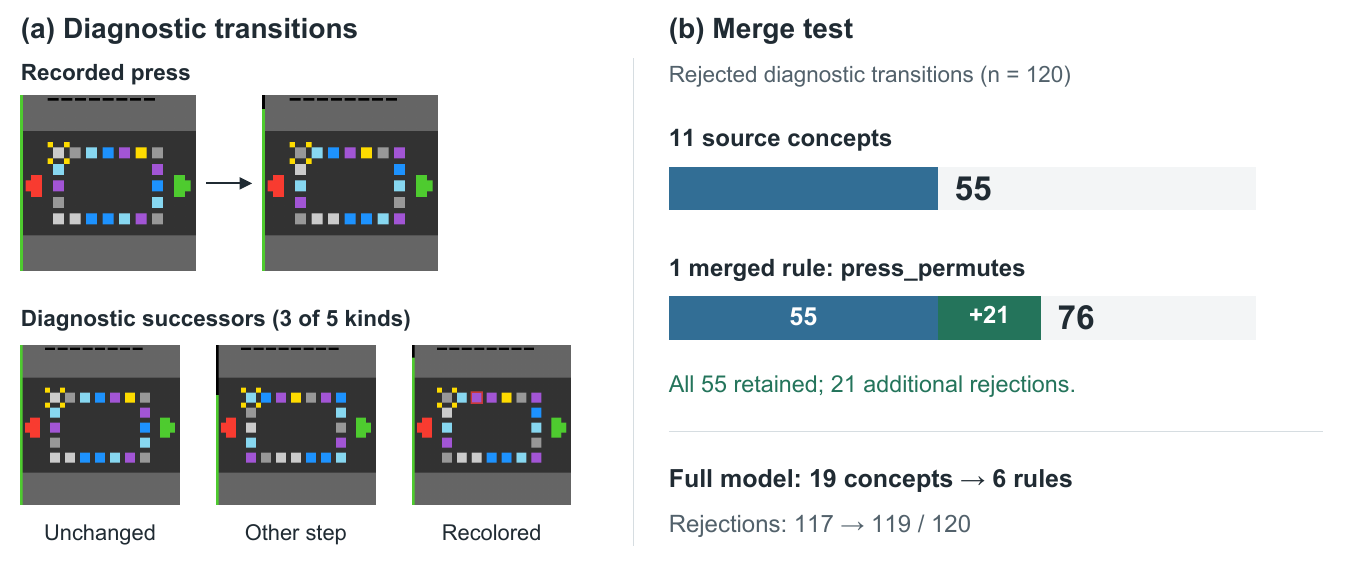}
\caption{Testing a merge on \code{lp85}. (a)~A recorded press and three of the five diagnostic constructions. (b)~On the same 120 diagnostic transitions, \code{press\_permutes} preserves all 55 rejections of its eleven source concepts and adds 21. Lack of unique rejections alone does not justify deletion: \code{dead\_click}, \code{board\_static} and \code{bar\_never\_drops} are retained based on further tests.}
\label{fig:merge}
\end{figure}

\apppar{Concepts as limits: presentations, transformations and their certificates.}{app:limits}
This paragraph states the definitions behind \S\ref{sec:consolidate}. They follow the logic of \citet{yuan2023succinct}: a concept is represented as a projective or inductive limit of simpler concepts; a representation is valid if it gives the same result as the concept on every task; representing the parts in turn yields a hierarchy whose leaves are basic concepts; and each node of the hierarchy is verified. That work uses formal limits in categories of presheaves and a verifier that may be a stronger model or a person. AWM implements the idea concretely: the concepts are the object types and rules of the model, the category is $\mathbf{Set}$ (in Lean, types and functions), which is complete and cocomplete, so every representation is an actual limit or colimit; the tasks are the checks on recorded transitions; and the verifier at each node is the Lean kernel. As in \citet{maclane1998categories} and \citet{kashiwara2006categories}, projective and inductive limits are limits and colimits.

\paragraph{Limits.}
\begin{definition}[Diagram, cone, limit]\label{def:A-limit}
A diagram is a functor $D\colon J\to\mathbf{Set}$ from a small category $J$. A cone over $D$ is a set $L$ with maps $\pi_j\colon L\to D(j)$, $j\in J$, such that $D(f)\circ\pi_i=\pi_j$ for every $f\colon i\to j$. It is a limit (projective limit) if for every cone $(L',\pi')$ there is a unique $u\colon L'\to L$ with $\pi_j\circ u=\pi'_j$ for all $j$. Dually, a cocone is a set $C$ with maps $\iota_j\colon D(j)\to C$ such that $\iota_j\circ D(f)=\iota_i$ for every $f\colon i\to j$, and it is a colimit (inductive limit) if every cocone $(C',\iota')$ factors as $\iota'_j=u\circ\iota_j$ through a unique $u\colon C\to C'$.
\end{definition}

In $\mathbf{Set}$, $\lim D$ is the set of compatible families $\{(x_j)\in\prod_j D(j): D(f)(x_i)=x_j\ \text{for all}\ f\colon i\to j\}$, and $\operatorname{colim}D$ is $\coprod_j D(j)$ modulo the equivalence relation generated by $(i,x)\sim(j,D(f)(x))$. The shapes used here are those of the standard examples: a product and a coproduct (a discrete diagram); a pullback (a cospan $\bullet\to\bullet\leftarrow\bullet$), which aligns two components on a shared quantity; an equalizer (a parallel pair), such as the subset $\{x\in X:\chi(x)=\mathsf{true}\}$ cut out by $\chi\colon X\to\mathsf{Bool}$ and the constant map $\mathsf{true}$; and a coequalizer, such as the quotient of $X$ by the equivalence relation generated by a relation $R\subseteq X\times X$, presented by the two projections $R\rightrightarrows X$.

Two consequences of the universal properties are used throughout. A map into a limit is the same as a compatible family of maps into its parts, and a map out of a colimit is the same as a compatible family of maps out of its parts. In particular, a rule on a coproduct of cases is determined by its restrictions to the cases, and a map out of a quotient is a map out of $X$ that is constant on $R$. Moreover, a limit or colimit is unique up to a unique isomorphism compatible with the cone or cocone: the quotient of color lists by permutation and the type of sorted color lists, both shown to be coequalizers of the same pair, are identified by the isomorphism that their universal properties provide.

\paragraph{Representations.}
\begin{definition}[Presentation]\label{def:A-presentation}
A presentation of an object type $T$ is a triple $(D_T,C_T,u_T)$: a diagram $D_T$ whose values are previously introduced types, a cone (or cocone) over $D_T$ with vertex $C_T$, and a witness $u_T$ of its universal property. The vertex $C_T$ is the type used for computation, so the universal property is established for that type itself rather than for a separate copy.
\end{definition}

Because a presentation refers only to previously introduced types, the presentations form a hierarchy. Its leaves are base types such as lists of colors, which are limits only trivially, of the one-object diagram on themselves. Every other object type is a limit or colimit of simpler types: a record without constraints is a product, and a type with constraints is an equalizer, a pullback or a quotient. Rules are represented in the same way. A rule defined by cases is a map out of a coproduct of cases, given by its restriction to each case, and the states or transitions that a condition such as the goal accepts form an equalizer. A rule given by a single formula is an ordinary definition, and a rule that the model does not represent stays in Python and is checked only on the corpus.

\begin{definition}[Case table]\label{def:A-cases}
A case table for a rule group with scope $C$ is a finite family of case predicates $(p_i)_{i\in I}$ on $C$. It satisfies the coproduct law if every element of $C$ satisfies exactly one $p_i$; then $C\cong\coprod_{i\in I}C_i$ with $C_i=\{x\in C:p_i(x)\}$. On the corpus we also require every case to occur. A rule defined case by case is the copairing $[r_i]_{i\in I}\colon\coprod_i C_i\to V$, the unique map whose restriction to each $C_i$ is $r_i\colon C_i\to V$.
\end{definition}

\paragraph{Validity under tasks.}
A representation is valid if replacing a concept by it changes the result of no task. The tasks of AWM are the checks that the model must pass on recorded data: the verdict of every rule on every recorded transition, the rejections on the diagnostic set $\mathcal{N}$, and the task contracts, which require that one step commute with the representation map $e$ from old to new states and that goal verdicts, checker verdicts and action costs be unchanged. A transformation is valid under these tasks if all verdicts and contracts are preserved and the rejections satisfy Eq.~\eqref{eq:unify}, so that $\mathcal{N}$ may only gain rejections.

\paragraph{Definitions and examples.}
A projective limit describes a concept by definition, through the components that support it, and an inductive limit describes it by examples, through what the examples share. The four transformations move along these two kinds of limits.
\begin{definition}[Category-guided transformation]\label{def:A-transform}
A transformation replaces components $S$ of the model by components $T$, together with a representation map $e$ from old to new states. It is one of the following.
\begin{itemize}
\item \emph{Decomposition}: a component $X$ is replaced by a diagram $D$ of simpler components together with a limit cone $(X,\pi)$ over $D$.
\item \emph{Composition}: a component $Y$ is defined from a product $Z$ of existing components and constraints given by subobjects $\varphi_1,\dots,\varphi_k\subseteq Z$, as their intersection, i.e.\ the limit of the inclusions $\varphi_i\hookrightarrow Z$. The composition is proper if no single $\varphi_i$ equals $Y$.
\item \emph{Generalization}: rules $r_i\colon C_i\to V$, $i\in I$, are replaced by the copairing $[r_i]\colon\coprod_i C_i\to V$, which a parameter $\theta\in I$ indexes as $r(\theta)$.
\item \emph{Specialization}, the converse of generalization: a component $X$ is replaced by components $X_i$ with $X\cong\coprod_i X_i$, or by a family $X(\theta)$ with a new parameter and the projection that forgets $\theta$.
\end{itemize}
It is accepted if (i) it is valid under the tasks above and the check of its kind passes, and (ii) the Lean statement of its universal construction compiles when the transformation is made. Deleting a component that is a function of others is recorded separately and is not one of the four transformations.
\end{definition}

The checks of each kind are: for decomposition, that $X\to\lim D\to X$ is the identity on recorded states; for composition, that $Y$ agrees with $\bigwedge_i\varphi_i$ and that each $\varphi_i$ admits a transition that $Y$ rejects; for generalization, the coproduct law and agreement with each $r_i$ on its case; for specialization, that merging the cases back reproduces the original verdicts and that fixing the new parameter to a constant produces errors.

\paragraph{Verification at each node.}
Every non-leaf node of the hierarchy is verified by the Lean kernel: every object type with constraints and every case table must come with a proof of its universal property. The kernel establishes that each node is the limit or colimit of its diagram; whether the representation describes the environment is established by the tasks. The two kinds of evidence therefore answer different questions, what the model entails and whether it agrees with the environment, and neither substitutes for the other. A claim for all well-formed states requires a Lean theorem that does not hold by definitional unfolding alone; for example, the equation stating that a rule defined by pattern matching equals the copairing of its branches holds by unfolding and certifies the form of the rule, not its content.

\paragraph{Action sequences.} A deterministic step $\delta\colon\mathcal{W}\times\mathcal{A}\to\mathcal{W}$ makes the free monoid $\mathcal{A}^*$ act on $\mathcal{W}$ on the right, by $w\cdot\varepsilon=w$ and $w\cdot(\bar a\,a)=\delta(w\cdot\bar a,a)$. With partial legality, each $\delta(-,a)$ is a morphism $\mathcal{W}\to\mathcal{W}$ in the Kleisli category of the option monad, and words act by partial maps \citep{moggi1991notions}. Facts about action sequences are not limits or colimits; one belongs to the model only if it is checked on every level of the corpus. In \code{lp85}, a word of clicks acts by the composite of the clicked buttons' permutations; Lean proves this for every board.

\paragraph{Lean interface.} Diagrams are given over finite quivers, since cones over a quiver coincide with cones over the free category it generates. The names \code{Cone}, \code{Cocone}, \code{IsLimit}, \code{IsColimit} follow Mathlib, with an implementation in Lean core. A final check compiles every module, requires the proofs listed under verification, keeps the computational layer independent of these proofs, forbids \code{sorry}, \code{admit} and \code{native\_decide}, and audits the axioms of every constant. For \code{lp85}, the declarations behind the examples of \S\ref{sec:rules2axioms} are in the ancillary files of this paper (\code{anc/lp85-lean}): \code{worldIsLimit} (the world as a pullback), \code{goalIsLimit} (the goal as an equalizer), and \code{onHit\_is\_desc} and \code{toPerm\_unique} (rules defined by cases as copairings).

\apppar{Lean encoding, certificates and transfer.}{app:lean}
Formalization starts from the consolidated objects and rules. AWM encodes object types as Lean \code{structure}, \code{inductive} or quotient declarations and primitive rules as executable \code{def} definitions. The presentations, case tables and transformation statements written during Structure \& Simplify to check each step (\S\ref{app:limits}) form an auxiliary layer that the model does not import. Consequences derivable from these definitions are proved as \code{theorem} statements in the derived layer. The agent assembles the modeled transition and goal rules in \code{Step.lean}, whose \code{step} function dispatches to the named rule definitions and returns a successor or a terminal outcome for actions admitted by the model's legality conditions. For example, in \code{lp85}, a click either leaves the colors unchanged or applies a button's permutation, after which the goal rule determines whether the level is won; color conservation is a theorem derived from these definitions. This shared transition function is used both to replay recorded trajectories and to reason about unobserved action sequences. In \code{lp85}, the two budget-display rules remain in Python; Lean replays the tile dynamics and goal condition, while color conservation follows from the permutation rule (Figure~\ref{fig:lp85full}).

On this basis, $\mathrm{everWon}(w,\bar a)$ holds when some legal prefix of $\bar a$ reaches a winning outcome. The two statements of \S\ref{sec:modelreasoning} are

{\small
\[
  \mathrm{Winnable}(w):\Leftrightarrow \exists\,\bar a.\
  \mathrm{everWon}(w,\bar a)=\mathsf{true},\qquad
  \mathrm{Unsolvable}(w):\Leftrightarrow \forall\,\bar a.\
  \mathrm{everWon}(w,\bar a)=\mathsf{false}.
\]}

For level $\ell$, the initial world $w^\ell_0$ is parsed from that level's first screen in the certified run, and the certified actions are $\bar a^\ell$. The reachability proof establishes $\mathrm{traceOk}(w^\ell_0,\bar a^\ell)$, i.e.\ that every computed state agrees with the recorded frame on the components the formal model represents, and $\mathrm{Winnable}(w^\ell_0)$ with $\bar a^\ell$ as witness. An impossibility proof establishes $\mathrm{Unsolvable}$ within a budget through a potential satisfying Eq.~\eqref{eq:potential}. Each theorem is closed by kernel reduction, its axioms lie within $\{\texttt{propext},\texttt{Classical.choice},\texttt{Quot.sound}\}$, and no proof uses \texttt{native\_decide}; compiled with Lean 4.34.0, the published \code{Game.lean} of every game compiles and its \code{consistent} and \code{faithful} theorems depend only on \texttt{propext}.

Definition~\ref{def:coverage} asks a sufficient cover to reproduce only the environment's optimal winning sequences under a fixed map $h$, rather than every transition as in an over-approximation \citep{aineto2024guarantees,clarke1994abstraction}. Let $\widehat G$ be the model's goal set.

\begin{theorem}[Transfer under sufficient cover]\label{thm:coverage-transfer}
Suppose $\widehat{\mathcal{M}}$ is a sufficient cover of $\mathcal E$ via $h$, and its initial state is $h(s_0)$. If $\mathrm{Unsolvable}_{\widehat{\mathcal{M}}}\bigl(h(s_0)\bigr)$, then no legal action sequence from $s_0$ reaches $G$.
\end{theorem}

\begin{proof}
Suppose some legal sequence from $s_0$ reaches $G$. Then the task can be won, so it has an optimal winning sequence $\bar a$. Since $\widehat{\mathcal{M}}$ is a sufficient cover, $\bar a$ is legal in $\widehat{\mathcal{M}}$ from $h(s_0)$, passes through the images of the real states, and ends in $\widehat G$. This is a model win, contradicting $\mathrm{Unsolvable}_{\widehat{\mathcal{M}}}(h(s_0))$.
\end{proof}

Full transition coverage is a special case: a model that reproduces every legal transition, with preserved goals and budget feasibility, reproduces every optimal winning sequence. For budgeted tasks, the optimal winning sequences are those of least cost, and the lemmas below give sufficient conditions that can be checked on engine-derived graphs.

For a fixed budget, let $\mathrm{Tr}_{X}(B)$ denote legal executions from the specified initial state that respect all budget checks. The model budget $\widehat B$ may use explicitly translated cost units.

\begin{lemma}[Budget-local execution transfer]\label{lem:budget-transfer}
If a map $\kappa:\mathrm{Tr}_{\mathcal E}(B)\to\mathrm{Tr}_{\widehat{\mathcal M}}(\widehat B)$ preserves winning outcomes, model impossibility at $\widehat B$ implies environment impossibility at $B$.
\end{lemma}
\begin{proof}
A real win would map to a budget-feasible model win, contradicting model impossibility.
\end{proof}

Unlike full coverage, this condition permits macro actions, stuttering and compressed detours. Goals, one-use resources and intermediate budget guards must still be preserved; final net cost alone is insufficient.

\begin{lemma}[Budget-local invariant transfer]\label{lem:invariant-transfer}
Let $D_B$ contain all real configurations reachable within budget $B$ from the specified initial state, including winning terminals. If a certificate predicate $I_B$ holds initially, is preserved by every budget-admissible real transition in $D_B$, and excludes every real goal in $D_B$, then the environment has no win at budget $B$.
\end{lemma}
\begin{proof}
Induction establishes $I_B$ at every feasible prefix; a winning prefix contradicts goal exclusion.
\end{proof}

For a potential $H$, summing its checked transition inequalities and terminal condition gives $H(q_0)\le\overline C$ for any winning execution, where $\overline C$ is its total in the proof's cost units. A real budget-$B$ win is excluded once cost accounting gives $\overline C\le\widehat B<H(q_0)$. Thus only the certified budget obstruction, not every intermediate modeled state, needs to be preserved.

A partial certificate may be completed, but the required initial bound and all transition and winning-boundary inequalities must then be verified, including any revised existing values. Appendix~\ref{app:transfer} gives the engine checks.

\section{Twenty-five games}
\label{app:games}
This appendix gives the per-game data behind \S\ref{sec:q1}, Table~\ref{tab:scores} and Table~\ref{tab:lean}, the test of frozen rules on unseen transitions, and the \code{lp85} proof sizes of \S\ref{sec:case}.

\apppar{Per-game results.}{app:scores}
Table~\ref{tab:formal} gives AWM's results and model for each game; Table~\ref{tab:scoresfull} gives the official scores of all systems.

\begin{table}[!htbp]
\centering\scriptsize
\setlength{\tabcolsep}{2.4pt}
\caption{AWM on each game. \emph{Score}: official score of the certified run. \emph{Cert.}: actions in the certified clears, which the score counts. \emph{Sent}: every action AWM sent to the engine, including uncertified sessions, failed attempts and resets. \emph{Human}: actions of the human baseline. \emph{Concepts}: concepts written during play, with those withdrawn in brackets. \emph{Rules (Lean)}: rules translated into Lean. \emph{Replayed}: levels whose certified winning run Lean replays from the first screen to the modeled goal.}
\label{tab:formal}
\begin{tabular}{@{}lrrrrrrrrrr@{}}
\toprule
Game & Lv. & Score & Cert. & Sent & Human & Concepts & Obj.\ types & Rules & Rules (Lean) & Replayed \\
\midrule
\code{ar25} & 8 & 100.0 & 253 & 253 & 748 & 17 (3) & 6 & 8 & 5 & 8/8 \\
\code{bp35} & 9 & 100.0 & 374 & 374 & 651 & 26 (4) & 3 & 11 & 10 & 9/9 \\
\code{cd82} & 6 & 100.0 & 70 & 661 & 171 & 44 (13) & 7 & 11 & 10 & 6/6 \\
\code{cn04} & 6 & 100.0 & 261 & 745 & 789 & 22 (1) & 4 & 9 & 8 & 6/6 \\
\code{dc22} & 6 & 100.0 & 1{,}037 & 2{,}664 & 1{,}228 & 48 (6) & 8 & 17 & 14 & 6/6 \\
\code{ft09} & 6 & 100.0 & 80 & 168 & 208 & 15 (6) & 6 & 7 & 4 & 6/6 \\
\code{g50t} & 7 & 100.0 & 288 & 811 & 879 & 20 (1) & 10 & 13 & 11 & 7/7 \\
\code{ka59} & 7 & 97.6 & 421 & 4{,}061 & 730 & 22 (8) & 9 & 11 & 10 & 7/7 \\
\code{lf52} & 10 & 100.0 & 905 & 905 & 1{,}339 & 47 (6) & 12 & 11 & 10 & 10/10 \\
\code{lp85} & 8 & 100.0 & 79 & 200 & 388 & 23 (4) & 6 & 6 & 4 & 8/8 \\
\code{ls20} & 7 & 100.0 & 368 & 1{,}467 & 776 & 22 (3) & 12 & 21 & 18 & 7/7 \\
\code{m0r0} & 6 & 100.0 & 215 & 486 & 1{,}107 & 26 (5) & 8 & 9 & 8 & 6/6 \\
\code{r11l} & 6 & 100.0 & 81 & 184 & 233 & 26 (4) & 6 & 9 & 8 & 6/6 \\
\code{re86} & 8 & 100.0 & 547 & 2{,}153 & 1{,}255 & 40 (4) & 7 & 12 & 8 & 8/8 \\
\code{s5i5} & 8 & 99.3 & 461 & 2{,}618 & 638 & 49 (9) & 5 & 19 & 17 & 8/8 \\
\code{sb26} & 8 & 100.0 & 124 & 424 & 213 & 20 (1) & 6 & 16 & 12 & 8/8 \\
\code{sc25} & 6 & 100.0 & 146 & 401 & 350 & 27 (1) & 5 & 18 & 11 & 6/6 \\
\code{sk48} & 8 & 100.0 & 263 & 1{,}407 & 1{,}070 & 39 (8) & 6 & 7 & 4 & 8/8 \\
\code{sp80} & 6 & 100.0 & 138 & 371 & 518 & 58 (23) & 3 & 10 & 6 & 6/6 \\
\code{su15} & 9 & 100.0 & 90 & 904 & 361 & 37 (15) & 5 & 16 & 12 & 9/9 \\
\code{tn36} & 7 & 100.0 & 125 & 292 & 317 & 33 (9) & 5 & 7 & 4 & 7/7 \\
\code{tr87} & 6 & 100.0 & 119 & 430 & 414 & 13 (3) & 6 & 7 & 4 & 6/6 \\
\code{tu93} & 9 & 100.0 & 187 & 566 & 462 & 17 (0) & 8 & 10 & 8 & 9/9 \\
\code{vc33} & 7 & 100.0 & 167 & 538 & 447 & 14 (0) & 9 & 10 & 7 & 7/7 \\
\code{wa30} & 9 & 100.0 & 649 & 1{,}521 & 1{,}843 & 37 (6) & 8 & 8 & 7 & 9/9 \\
\midrule
Total & 183 & 99.9 & 7{,}448 & 24{,}604 & 17{,}135 & 742 (143) & 170 & 283 & 220 & 183/183 \\
\bottomrule
\end{tabular}
\end{table}

\begin{table}[!htbp]
  \centering
  \caption{Official scores on the 25 public ARC-AGI-3 games. Baselines are the
  official ARC~Prize results: Claude Opus~5 (High) and GPT-6 Astra (Max) under the
  standard harness, which provides no code execution, and GPT-6 Astra (Max) under
  its provider-adapter harness. AWM is scored on its certified run.
  Rodionov's agent from its ARC Prize scorecard~\citep{rodionov2026need}; Memento~3 as reported in~\citep{zhao2026memento3}.
  }
  \label{tab:scoresfull}
  \setlength{\tabcolsep}{5.7pt}
  \footnotesize
  \begin{tabular}{@{}lrrrrrrr@{}}
    \toprule
    & & \multicolumn{6}{c}{Official score (max 100)} \\
    \cmidrule(l){3-8}
    Game & Lv. & \shortstack[r]{Claude\\Opus~5} & \shortstack[r]{GPT-6\\Astra} & \shortstack[r]{GPT-6 Astra\\(adapter)} & \shortstack[r]{Rodionov\\\citeyearpar{rodionov2026need}} & \shortstack[r]{Memento~3\\\citep{zhao2026memento3}} & \shortstack[r]{\textbf{AWM}\\\textbf{(ours)}} \\
    \midrule
    \code{ar25} & 8 & 100.0 & 100.0 & 100.0 & 100.0 & 100.0 & \textbf{100.0}\\
    \code{bp35} & 9 & 0.0 & 2.2 & 100.0 & 100.0 & 100.0 & \textbf{100.0}\\
    \code{cd82} & 6 & 0.6 & 82.9 & 100.0 & 100.0 & 100.0 & \textbf{100.0}\\
    \code{cn04} & 6 & 47.6 & 100.0 & 100.0 & 100.0 & 100.0 & \textbf{100.0}\\
    \code{dc22} & 6 & 44.8 & 100.0 & 100.0 & 100.0 & 100.0 & \textbf{100.0}\\
    \code{ft09} & 6 & 100.0 & 100.0 & 100.0 & 100.0 & 100.0 & \textbf{100.0}\\
    \code{g50t} & 7 & 0.0 & 0.0 & 100.0 & 100.0 & 100.0 & \textbf{100.0}\\
    \code{ka59} & 7 & 10.4 & 75.0 & 100.0 & 100.0 & 100.0 & \textbf{97.6}\\
    \code{lf52} & 10 & 1.8 & 10.9 & 100.0 & 100.0 & 100.0 & \textbf{100.0}\\
    \code{lp85} & 8 & 100.0 & 100.0 & 100.0 & 100.0 & 100.0 & \textbf{100.0}\\
    \code{ls20} & 7 & 5.4 & 83.5 & 100.0 & 100.0 & 100.0 & \textbf{100.0}\\
    \code{m0r0} & 6 & 28.6 & 71.4 & 100.0 & 100.0 & 100.0 & \textbf{100.0}\\
    \code{r11l} & 6 & 100.0 & 100.0 & 100.0 & 100.0 & 100.0 & \textbf{100.0}\\
    \code{re86} & 8 & 58.3 & 77.8 & 100.0 & 100.0 & 100.0 & \textbf{100.0}\\
    \code{s5i5} & 8 & 100.0 & 100.0 & 100.0 & 100.0 & 100.0 & \textbf{99.3}\\
    \code{sb26} & 8 & 77.8 & 100.0 & 100.0 & 100.0 & 100.0 & \textbf{100.0}\\
    \code{sc25} & 6 & 0.0 & 47.6 & 100.0 & 84.2 & 100.0 & \textbf{100.0}\\
    \code{sk48} & 8 & 8.3 & 19.7 & 100.0 & 100.0 & 100.0 & \textbf{100.0}\\
    \code{sp80} & 6 & 56.3 & 100.0 & 100.0 & 100.0 & 100.0 & \textbf{100.0}\\
    \code{su15} & 9 & 10.4 & 27.4 & 100.0 & 100.0 & 100.0 & \textbf{100.0}\\
    \code{tn36} & 7 & 6.0 & 53.6 & 100.0 & 90.1 & 100.0 & \textbf{100.0}\\
    \code{tr87} & 6 & 47.6 & 100.0 & 100.0 & 100.0 & 100.0 & \textbf{100.0}\\
    \code{tu93} & 9 & 3.6 & 3.1 & 100.0 & 100.0 & 100.0 & \textbf{100.0}\\
    \code{vc33} & 7 & 98.8 & 73.3 & 100.0 & 100.0 & 100.0 & \textbf{100.0}\\
    \code{wa30} & 9 & 10.6 & 80.0 & 100.0 & 100.0 & 100.0 & \textbf{100.0}\\
    \midrule
    Mean & & 40.7 & 68.3 & 100.0 & 99.0 & 100.0 & \textbf{99.9}\\
    Games at 100 & & 5 & 10 & 25 & 23 & 25 & \textbf{23}\\
    \bottomrule
  \end{tabular}
\end{table}

\apppar{Frozen rules on unseen transitions.}{app:unseen}
We evaluate whether learned rules generalize to unseen
screens and detect changes in game mechanics.
For \code{lp85} and \code{ft09}, we freeze the concepts
recorded during play (19 and 9, respectively) and the
rules after Structure \& Simplify (6 and 7).
We collect transitions from random presses (seeds 0--9),
retaining only states whose screens never appeared in
the recorded gameplay.

Each game is evaluated on new trajectories from the
original levels (T1) and two modified levels, whose
winnability is verified by exhaustive search.
In \code{lp85}, the modifications preserve the mechanics
but change the layout: U moves one goal tile, making
level~6 unwinnable, while D adds another swap that
restores winnability.
In \code{ft09}, the modifications change the hidden
click mechanic: FU replaces level~2's click pattern
with a plus shape, making it unwinnable, while FD adds
a one-pixel change that restores winnability.

A rule set accepts a transition if no rule judges it
false, and rejects it if at least one rule judges it
false. We also generate invalid transitions using
five corruption operators to test whether the rules
can detect inconsistencies.
Table~\ref{tab:heldout} reports the results before
and after Structure \& Simplify. The frozen rules
generalize to modified layouts in \code{lp85},
while detecting all tested transitions as inconsistent
after the mechanic changes in \code{ft09}.

\begin{table}[!htbp]
\centering\footnotesize
\setlength{\tabcolsep}{3pt}
\caption{Frozen-rule evaluation on unseen transitions,
before / after Structure \& Simplify (\%).
\code{lp85} U/D test generalization to modified layouts,
while \code{ft09} FU/FD test detection of changed mechanics.
``Wrong rejected'' reports the range over the four
corruption operators for which both rule sets can read
both frames. The 0\% acceptance on FU/FD indicates
that all tested transitions violate the frozen rules:
the ``nothing else changes'' rule detects clicks
changing 110 or 146 cells instead of 38.}
\label{tab:heldout}
\begin{tabularx}{\linewidth}{@{}Xrrrr@{}}
\toprule
Set & $n$ & Transitions accepted & Wrong rejected & Next frame predicted \\
\midrule
lp85 T1 (new trajectories) & 2{,}073 & 100 / 100 & 99.5--100 / 99.5--100 & 99.95 / 99.95 \\
lp85 tile swap U & 2{,}080 & 100 / 100 & 100 / 100 & 100 / 100 \\
lp85 tile swap D & 2{,}080 & 100 / 100 & 100 / 100 & 100 / 100 \\
ft09 T1 (new trajectories) & 2{,}026 & 99.0 / 100 & 99.0--100 / 100 & --- \\
ft09 pattern change FU & 2{,}015 & 0 / 0 & --- & --- \\
ft09 pattern change FD & 2{,}016 & 0 / 0 & --- & --- \\
\bottomrule
\end{tabularx}
\end{table}

\paragraph{\code{ls20}.}
On \code{ls20}, the concepts written during play accept 67--70\% of valid transitions on unseen screens; after Structure \& Simplify, the frozen rules accept 99.47--99.95\% and reject 99.1--99.3\% of corrupted transitions, against 71.8--100\% for the concepts (Table~\ref{tab:ls20}). The gap comes from wall bumps: the concepts \code{energy\_tick} and \code{energy\_cost} charge energy for a blocked move only on the two levels where play recorded one, while the rule \code{energy\_cost} charges every wall bump the level's step cost, which holds on screens the agent never saw.

\begin{table}[!htbp]
\centering
\caption{Frozen concepts and rules on two \code{ls20} test sets of previously unseen screens (\%). T1: new trajectories on the original levels; T2V: level~1 with its switch moved one cell. Corrupted rejected: range over the four corruption operators for which both sides can read both frames.}
\label{tab:ls20}
\small
\begin{tabular}{@{}lrcccc@{}}
\toprule
 & & \multicolumn{2}{c}{Valid accepted} & \multicolumn{2}{c}{Corrupted rejected} \\
\cmidrule(lr){3-4}\cmidrule(lr){5-6}
Test set & $n$ & Concepts & Rules & Concepts & Rules \\
\midrule
T1 & 2{,}017 & 67.0 & 99.95 & 75.3--100 & 99.1--99.3 \\
T2V & 2{,}074 & 70.3 & 99.47 & 71.8--100 & 99.1--99.2 \\
\bottomrule
\end{tabular}
\end{table}

\begin{figure}[!htbp]
\centering
\includegraphics[width=\linewidth]{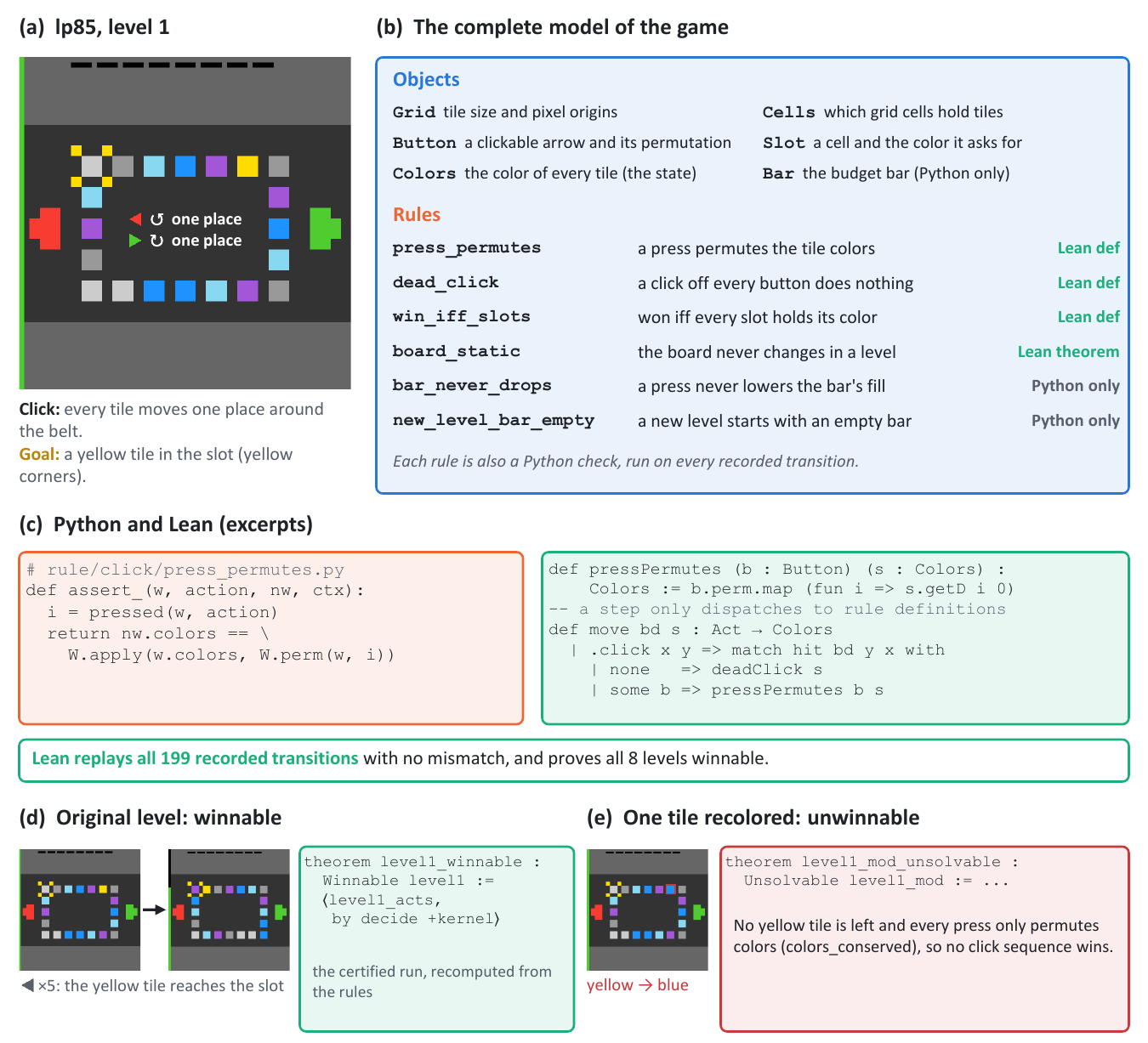}
\caption{The complete algebraic world model of an ARC-AGI-3 game. (a)~Level~1 of \code{lp85}: 20 colored tiles form a closed belt, and each click on an arrow moves every tile one place around the belt, counterclockwise for the left arrow and clockwise for the right; the level is won when a yellow tile stands in the slot marked by yellow corners. (b)~All objects and rules of the model, with the form each takes in Lean; the two rules about the budget bar stay in Python. (c)~Excerpts: one rule as its Python check, the same rule as a Lean definition, and the Lean step, which only dispatches to rule definitions. Lean replays every recorded transition of the game with no mismatch. (d)~On the original level, Lean checks a winning witness, the certified run. (e)~With the only yellow tile recolored, the conservation law of \code{press\_permutes} shows that no click sequence wins. Both statements are theorems about the model.}
\label{fig:lp85full}
\end{figure}

\apppar{\code{lp85}: the complete model and proof size.}{app:e3}
Figure~\ref{fig:lp85full} shows the complete model of
\code{lp85}. To evaluate the effect of Structure \&
Simplify on proof construction, we compare model A,
the eleven concepts recorded during play, with model B,
the single merged \code{press\_permutes} rule.
We ask Claude Opus~5.5 to establish the solvability
(with a winning click sequence) or unsolvability
(with a Lean proof) of 34 tasks under each model.
Each task is run three times per model, yielding
204 runs in total, and every verdict is checked
against exhaustive search on the engine.

Table~\ref{tab:lp85proof} reports proof success and
complexity. Model B succeeds in all 102 runs, compared
with 99 for A, and generally requires fewer auxiliary
lemmas, shorter proofs, and less Lean checking time.

\begin{table}[!htbp]
\centering\footnotesize
\setlength{\tabcolsep}{3pt}
\caption{Proof success and complexity for 34 \code{lp85}
tasks using the original concepts (A) or the merged
rule (B). Verified counts successful runs;
auxiliary lemmas, proof lines, and kernel checking
time are medians over successful runs.
The macro-average is the unweighted mean of the
six group medians. All values are reported as A / B.}
\label{tab:lp85proof}
\begin{tabularx}{\linewidth}{@{}Xrcccc@{}}
\toprule
Instances & $n$ &
\shortstack{Verified\\A / B} &
\shortstack{Aux. lemmas\\A / B} &
\shortstack{Proof lines\\A / B} &
\shortstack{Kernel time (s)\\A / B} \\
\midrule
original levels (winnable)
& 24 & 24 / 24 & 0 / 0 & 4.5 / 4 & 1.00 / 0.83 \\
goal colour removed (unwinnable)
& 24 & 24 / 24 & 6 / 3 & 75.5 / 28 & 0.90 / 0.80 \\
level~6 swaps (U and D)
& 6 & 6 / 6 & 2 / 1.5 & 36.5 / 24 & 1.69 / 1.22 \\
large layouts (winnable)
& 18 & 16 / 18 & 0 / 0 & 5.5 / 4 & 1.95 / 1.00 \\
large layouts, altered (unwinnable)
& 18 & 17 / 18 & 12 / 4 & 117 / 49.5 & 2.82 / 2.41 \\
ring spacing (unwinnable)
& 12 & 12 / 12 & 12 / 6.5 & 111 / 62.5 & 4.27 / 3.02 \\
\midrule
Macro avg. of medians
& -- & -- & 5.33 / 2.50 & 58.33 / 28.67 & 2.11 / 1.55 \\
\textbf{All runs}
& \textbf{102} & \textbf{99 / 102} & -- & -- & -- \\
\bottomrule
\end{tabularx}
\end{table}

No run produces an incorrect verdict: A has two
timeouts and one run ending without an answer.
On winnable instances, both models establish
reachability using a winning click sequence and
\code{decide}.

\apppar{Checking a model-level proof against the engine.}{app:engine-proof}
This check concerns the level-6 tile-swap variant U of Appendix~\ref{app:unseen}. We export its engine-derived, budget-independent reachable configuration graph as a Lean table, with 7,500 configurations and 52,500 transitions for the seven nontrivial button actions. A Lean proof checks each transition against the premises of the model-level unsolvability theorem, so the conclusion transfers to the exported graph. The proof uses only Lean's standard axioms. Extracting the graph from the game engine remains outside the Lean kernel's guarantee.

\section{Sealed levels}
\label{app:sc25}
\begin{figure}[t]
\centering
\includegraphics[width=\linewidth]{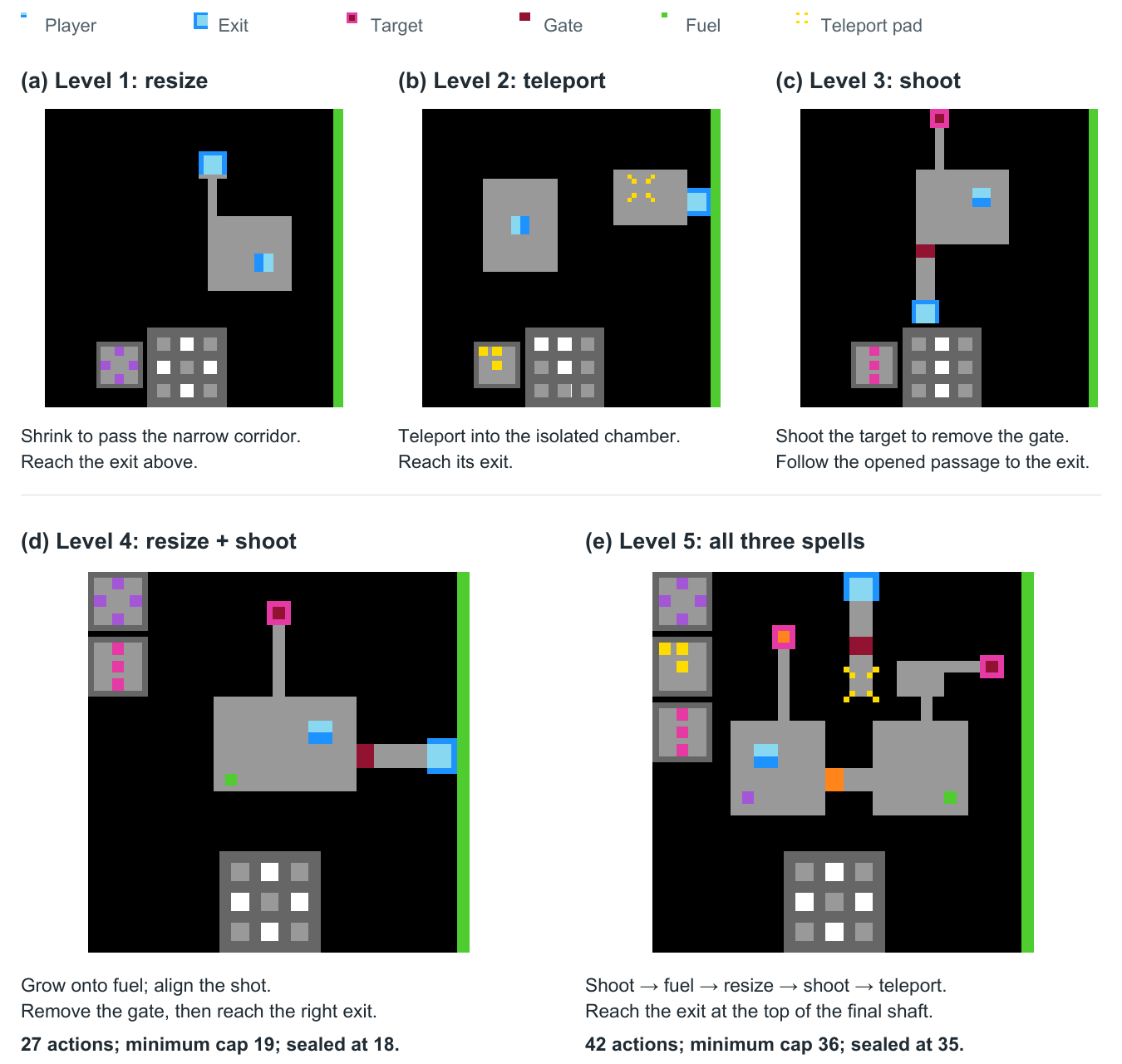}
\caption{The three introductory levels and two sealed evaluation levels of \code{sc25}, rendered from the engine in the official palette. Panels (a--c) isolate resizing, teleportation and shooting; panels (d--e) combine them. The annotations show each objective and the main stages of a minimum-budget winning route. The sealed caps, 18 and 35, are one unit below the required 19 and 36.}
\label{fig:sc25levels}
\end{figure}

The sealed tasks give a known winning-cost threshold against which to assess the reasoning in \S\ref{sec:sc25}. A proof can be valid under a learned model even when that model misstates the game. We therefore examine the initial bounds and their underlying rules, then follow the corrections through renewed proofs and checks against the engine. The final studies test how proof dependencies help locate a faulty rule and what the verification procedure itself can detect.

\needspace{12\baselineskip}
\apppar{Construction and engine ground truth.}{app:truth}
Each task is obtained by lowering a solvable level's budget to one unit below its minimum feasible cap $B^\star$. This gives both a winning witness at $B^\star$ and a known impossible instance at $B^\star-1$ (Table~\ref{tab:seals}).

\begin{table}[H]
\centering\footnotesize
\setlength{\tabcolsep}{4pt}
\caption{Reference thresholds for the sealed tasks. $B^\star$ is the minimum feasible cap; ``gross'' adds back the fuel refund of 10. Each witness wins at $B^\star$ and fails at $B^\star-1$, while exhaustive search finds no winning route at the sealed cap from any opening facing. Route lengths count actions, while costs use the engine's counter units.}
\label{tab:seals}
\begin{tabular}{@{}lrrrr@{}}
\toprule
Level & Public cap & Sealed cap & $B^\star$ (gross) & Optimal route \\
\midrule
4 & 35 & 18 & 19 (29) & 27 actions \\
5 & 65 & 35 & 36 (46) & 42 actions \\
\bottomrule
\end{tabular}
\end{table}

\code{sc25} is a deterministic, top-down puzzle game in which a block must reach an exit through walls and gates. Direction commands set its facing and move it four pixels when large ($4\times4$) or two when small ($2\times2$); a blocked large move also tries a two-pixel step. Matching a pattern on the $3\times3$ panel casts a spell: resize, teleport to a size-dependent destination, or fire a straight ray at a target to remove its gates. Walls and gates stop the ray. Direction commands and panel clicks cost one unit, spell completion costs one more, and fuel refunds ten. The winning move is free; exceeding the cap ends the attempt. The right-edge bar displays the remaining budget.

We rearrange the public levels' wall masks and object placements while preserving the transition code, existing sprite definitions, spell patterns and cost rules. Before sealing, available spells and budget caps match the public levels. Figure~\ref{fig:sc25levels} shows the resulting tasks. Levels~1--3 introduce shrinking through a narrow corridor, teleporting into an isolated exit chamber, and shooting to open a gated passage. Their caps are 50, 25 and 50; replayed witnesses use 15, 6 and 14 actions.

Level~4 combines a target at the end of a two-pixel-wide shaft, a gate across the right exit passage, and fuel. The large block's initial movement grid misses the shaft. A minimum-budget route shrinks and regrows over the fuel, obtaining the refund and changing its alignment, then lines up a shot and exits through the opened gate. Its 27 actions require a cap of 19 (gross cost 29). Level~5 has two target--gate groups, fuel, a narrow connector and separate teleport destinations for each size. A minimum-budget route shoots the first target, collects fuel, shrinks through the connector, regrows to shoot the second target, and teleports as a large block into the exit shaft. Its 42 actions require 36 (gross cost 46). These routes minimize the required budget, rather than the unconstrained number of actions.

The witnesses establish that caps 19 and 36 suffice. To establish that lower caps cannot win, we search the sealed engine before any agent run, covering all four opening facings and all 13 actions (four moves and nine panel cells). The reachable configurations number 905{,}587 on level~4 and 1{,}423{,}609 on level~5. For each configuration, the search retains the least counter value, since a lower counter leaves at least as much budget for every continuation. We audit all six counter updates and both cap checks in the source. To check for effective clicks omitted by the action set, we also scan the whole $64\times64$ screen on 30{,}001 sampled frame classes per level; clicks outside the nine panel cells change none of the 13 actions' outcomes. Editing the cap field changes only the budget, so the winning and sealed instances share the same dynamics.

\apppar{Agents and protocol.}{app:configs}
The comparison tests whether agents can identify the budget obstruction and justify it from interaction. Three configurations play each sealed level, three runs each. AWM uses Claude Opus~5.5 with the AWM procedure, a generic Lean skeleton and tools for recording predictions and checking rules; it receives no game rules. Claude Opus~5.5 and GPT-6 Astra without AWM receive the task, game interface and ordinary coding-agent tools.

Each run has up to eight reset-delimited attempts, ending at game over or recorded abandonment. After each attempt, the interview in Appendix~\ref{app:prompts} records the agent's verdict and argument. The engine reference thresholds are withheld: arenas have opaque names, contain no ground-truth file, and begin with cross-run memory cleared. This keeps the agent's reported bound separate from the reference search used to evaluate it.

\apppar{Proof validity and bound accuracy.}{app:mainruns}
The per-run results explain why a model proof and an accurate account of the game are scored separately. All six AWM runs produce closed Lean impossibility proofs, but two initially report overly high thresholds: 20 on level~4 and 47 on level~5. Their proofs certify the learned rules; the numerical errors reveal that those rules exclude cheaper behavior allowed by the engine. Table~\ref{tab:runs} also records the baseline verdicts, including an impossibility claim later withdrawn. These individual outcomes underlie Table~\ref{tab:verdicts} and motivate the repair round below.

\emph{First claim} records the attempt of the earliest impossibility verdict, even if later withdrawn; it is not proof-completion time. \emph{Bound} is the least winning cost reported before feedback, in the run's own units for runs without AWM. The reference $B^\star$ is 19 on level~4 and 36 on level~5. A bound counts as correct if it equals $B^\star$ or differs only by charging the otherwise free winning move: 37 on level~5, or 42 actions against a budget of 41 in \code{fbac96}. This accounting allowance applies to every agent.

\begin{table}[!htbp]
\centering\footnotesize
\setlength{\tabcolsep}{3pt}
\caption{Per-run verdicts, reported bounds and model proofs for the 18 runs summarized in Table~\ref{tab:verdicts}. ``Correct'' evaluates the reported bound against the engine; ``Lean proof'' records a closed, kernel\mbox{-}checked impossibility theorem under the run's learned model. The text specifies the timing and cost-accounting conventions.}
\label{tab:runs}
\begin{tabular}{@{}llllllcc@{}}
\toprule
Level & Agent & Run & First claim & Final verdict & Bound & Correct & Lean proof \\
\midrule
4 & AWM & \code{9ec8d4} & 2 & impossible & 20 & no & yes \\
4 & AWM & \code{69fe87} & 1 & impossible & 19 & yes & yes \\
4 & AWM & \code{2a3cc5} & 2 & impossible & 19 & yes & yes \\
4 & Claude Opus~5.5 & \code{2f1a1c} & --- & undecided & --- & no & no \\
4 & Claude Opus~5.5 & \code{d267df} & --- & undecided & --- & no & no \\
4 & Claude Opus~5.5 & \code{1dd6e6} & --- & undecided & --- & no & no \\
4 & GPT-6 Astra & \code{44ec53} & 8 & undecided & 22 & no & no \\
4 & GPT-6 Astra & \code{4d3b31} & 8 & impossible & 22 & no & no \\
4 & GPT-6 Astra & \code{98893b} & --- & undecided & --- & no & no \\
\midrule
5 & AWM & \code{284967} & 3 & impossible & 37 & yes & yes \\
5 & AWM & \code{4f4628} & 2 & impossible & 36 & yes & yes \\
5 & AWM & \code{977714} & 2 & impossible & 47 & no & yes \\
5 & Claude Opus~5.5 & \code{339b0a} & 4 & impossible & 47--49 & no & no \\
5 & Claude Opus~5.5 & \code{202da3} & 6 & impossible & 43 & no & no \\
5 & Claude Opus~5.5 & \code{fbac96} & 6 & impossible & 42 (budget 41) & yes & no \\
5 & GPT-6 Astra & \code{2cdca5} & --- & undecided & --- & no & no \\
5 & GPT-6 Astra & \code{3eb226} & --- & undecided & --- & no & no \\
5 & GPT-6 Astra & \code{b3830e} & --- & undecided & --- & no & no \\
\bottomrule
\end{tabular}
\end{table}

\apppar{Repairing the rules and proving again.}{app:repair}
The two overestimated AWM bounds provide a test of whether an explicit model can be corrected after a valid proof has been obtained. After the eight-attempt round, those runs and the four baseline runs with wrong bounds receive one piece of feedback: the engine permits a cheaper win than their analysis allows. The message reveals no rule, cost or route. AWM additionally receives a self-audit instruction to relax rules consistently with the record, test disagreements in the game, revise and prove again; the baseline message states the same discrepancy logic in plain language. Each run gets at most four further attempts.

The corrections change specific premises of the proofs (Table~\ref{tab:feedback}). On level~4, growing onto fuel collects it, reducing AWM's bound from 20 to 19; Figure~\ref{fig:l4level} shows the missed interaction and the cheaper route. On level~5, terrain rather than limited range blocks a shot, reducing the bound from 47 to 37 under the run's convention of charging the winning move, equivalent to the engine's 36. Both repairs produce new closed Lean proofs. Astra corrects both numerical bounds without producing a formal proof, while Claude Opus~5.5 corrects neither. Thus numerical revision and renewed formal certification remain distinct outcomes of the same feedback.

\begin{table}[!htbp]
\centering\footnotesize
\setlength{\tabcolsep}{3pt}
\caption{Outcomes after feedback that a cheaper win exists. The revised-rule column identifies the premise changed by AWM before proving again. Bounds of runs without AWM remain in their own units, with the budget they state. Attempts count additional attempts in the repair round.}
\label{tab:feedback}
\begin{tabularx}{\linewidth}{@{}>{\raggedright\arraybackslash}p{31mm}cc>{\raggedright\arraybackslash}X>{\raggedright\arraybackslash}p{24mm}c@{}}
\toprule
Run & Level & Attempts & Rule revised & Bound after & Lean proof \\
\midrule
AWM \code{9ec8d4} & 4 & 4 & growing onto fuel collects it & 19 & yes \\
AWM \code{977714} & 5 & 4 & terrain, not range, stops the shot & 37 & yes \\
GPT-6 Astra \code{44ec53}, \code{4d3b31} & 4 & $\le$4 & --- & 19 (both) & no \\
Claude Opus~5.5 \code{339b0a}, \code{202da3} & 5 & $\le$4 & --- & 42 (budget $\sim$40), 68 (budget 64) & no \\
\bottomrule
\end{tabularx}
\end{table}

\begin{figure}[H]
\centering
\includegraphics[width=\linewidth]{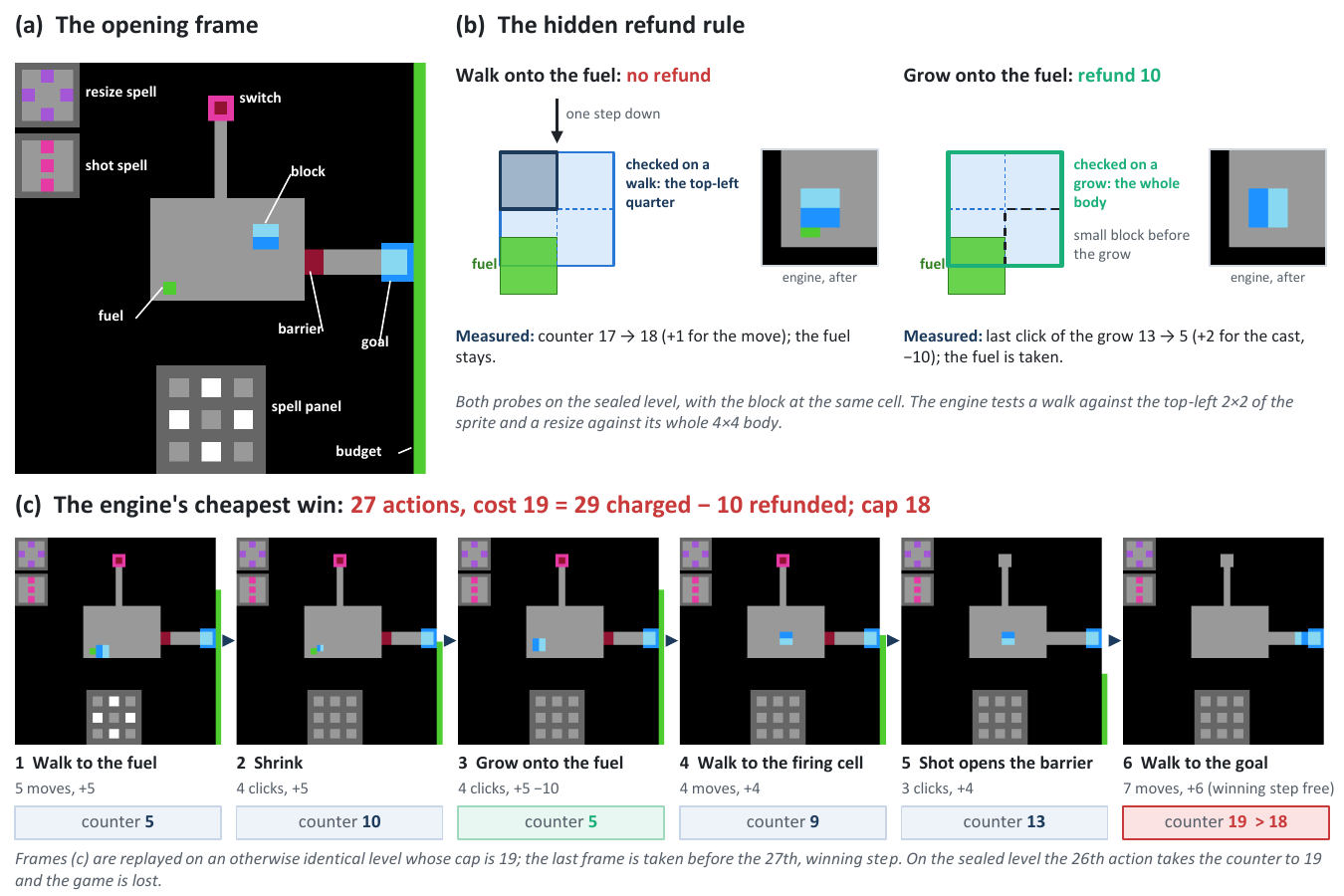}
\caption{Why the level-4 model needed revision. (a)~The opening. (b)~Walking checks only the large block's top-left quarter for fuel, whereas growing checks its full footprint and can collect fuel under the bottom-left quarter, refunding~10. Missing this distinction led one AWM run to report~20. (c)~The cheaper witness uses 27 actions and requires~19. It is replayed at cap~19; at the sealed cap~18, action~26 ends the game. All frames come from the engine.}
\label{fig:l4level}
\end{figure}

\apppar{Transferring budget certificates to the engine.}{app:transfer}
We check each model proof's invariant or potential certificate on engine-generated transitions, using completed \code{2a3cc5} V4, feedback-repaired \code{9ec8d4} and \code{977714}, and the other runs' final models. Five proofs extend to the games (Table~\ref{tab:budget-transfer}).

\begin{table}[!htbp]
\centering\footnotesize
\setlength{\tabcolsep}{4pt}
\caption{Budget-local certificate audit at the fixed opening and sealed caps, 18 on L4 and 35 on L5. Pass includes the listed post-hoc treatment; it does not assert full coverage.}
\label{tab:budget-transfer}
\begin{tabularx}{\linewidth}{@{}clXc@{}}
\toprule
Level & Run / snapshot & Transfer argument & Result \\
\midrule
4 & \code{9ec8d4}, repaired & Original $NN/MM$ values; panel-progress lifting & Pass \\
4 & \code{69fe87} & Original $V$ and domain; panel credit & Pass \\
4 & \code{2a3cc5}, V4 & Original arrival-cost closure invariant & Fails \\
5 & \code{284967} & Extend potential; bridge macro prices and fuel & Pass \\
5 & \code{4f4628} & Extend weighted certificate; translate budget & Pass \\
5 & \code{977714}, repaired & Extend certificate using real growth and costs & Pass \\
\bottomrule
\end{tabularx}
\end{table}

\paragraph{Procedure.}
From the fixed opening, the unchanged engine generates every transition within the sealed cap, 18 on L4 and 35 on L5, and each successor is taken from the engine rather than from the learned step function. For each certificate we check the conditions of Lemma~\ref{lem:invariant-transfer}: it holds initially, is preserved by every budget-admissible real transition, and excludes every real win. For a potential, these are the inequalities of Eq.~\eqref{eq:potential} and the winning boundary after translating the certificate's cost units into the game's budget, including intermediate budget checks, refunds and the free winning move. On L4, both passing certificates hold with their original values once prepaid panel clicks are credited in the budget accounting. On L5, the learned models omit some real successors: a blocked move that turns in place, the footprint freed by a destroyed target, and further placements of growth. None of these transitions occurs on a cheapest winning route of the game, so they do not affect the impossibility proofs: the certificates are completed on them while keeping their original initial bounds, and every inequality is then verified again, so the certified lower bound holds in the game.

\paragraph{Failure.}
The invariant of \code{2a3cc5} V4 does not transfer: the model lets the block collect fuel where the game does not and misses a refund, and the engine produces real transitions on which the invariant fails, as checked with the original Lean definitions.

\paragraph{Scope.}
These are post-hoc certificate audits, not additional agent-produced proofs. Archived Lean definitions supply the original values; offline programs check graph extraction and engine correspondence, outside an end-to-end Lean verification of the engine. Their fixed-opening scope is distinct from the four-opening ground-truth search of \S\ref{app:truth}.

\apppar{Using proof dependencies to find model errors.}{app:e1}
The repairs above corrected rules used by the original proofs. We
now test whether knowing which rules a proof depends on helps an
agent find errors in its world model.

We start from a reference model that agrees with the game under
the abstractions of \S\ref{app:truth}, and deliberately make one
rule incorrect. We select 14 such errors on L4 and 7 on L5 that
change the model's minimum winning cost, and repeat each experiment
with three seeds.

Claude Opus~5.5 receives the faulty model, a closed Lean proof of
impossibility under that (faulty) model, and access to the game. Although the proof is valid
under the model's rules, those rules may disagree with the game.
Claude must find an actual transition where the model predicts
a different outcome from the game, within the faulty rule's scope.
Simply naming the faulty rule does not count as success.

We compare three conditions. In C1, Claude receives the model's
rules and their recorded evidence, without an explicit list of
proof dependencies. In C2, it additionally receives the list
of rules used by the proof, but is not told which one is wrong.
In R, it is directly told the faulty rule. We also include
random exploration (C0) and exhaustive engine search (EX) as
references. Each agent run has a budget of 50 game actions on
L4 or 5,000 on L5.

On L4, C1 finds a conflicting transition in 16 of 42 runs.
Providing proof dependencies (C2) increases this to 29 of 42.
Even when directly told the faulty rule (R), Claude succeeds
in only 30 of 42 runs. Thus, proof dependencies nearly match
the benefit of knowing the faulty rule, without revealing it.
L5 is harder: C1 succeeds in 1 of 21 runs, C2 in 3 of 21,
and R in 5 of 21.

Across both levels, C2 finds counterexamples in 32 of 63 runs,
compared with 17 of 63 for C1. On 12 error-free controls --- the
same procedure run on an unmodified reference model, with no
injected error --- C2 raises no false alarms, compared with 0.33
per run for C1.
Table~\ref{tab:e1main} additionally reports each condition's median
actions and tokens to discovery, how often it names the faulty rule
in the conflict log or final model, and the C0 (random) and EX
(exhaustive) references.

\begin{table}[!htbp]
\centering\footnotesize
\setlength{\tabcolsep}{3pt}
\caption{Counterexample discovery with injected rule errors.
Each of the 14 L4 and 7 L5 errors is tested with three seeds.
\emph{Found}: runs that discover a conflicting transition within
the action budget. \emph{Median}: Kaplan--Meier median actions
to discovery, censored at the budget. \emph{Named}: runs that
identify the faulty rule in the conflict log or final model.
EX has no action budget.}
\label{tab:e1main}
\begin{tabular}{@{}lccccccc@{}}
\toprule
Arm & L4 found & L4 median & L5 found & L5 median &
\shortstack{Found\\(L4 + L5)} & Named & Tokens (median) \\
\midrule
C0 random & 12/42 & $>50$ & 0/21 & $>5000$ & 12/63 & --- & --- \\
C1 per-rule evidence & 16/42 & $>50$ & 1/21 & $>5000$ & 17/63 & 35/63 & 2.35M \\
C2 proof dependencies & 29/42 & 27 & 3/21 & $>5000$ & 32/63 & 50/63 & 3.29M \\
R faulty rule named & 30/42 & 26 & 5/21 & $>5000$ & 35/63 & 61/63 & 1.98M \\
EX exhaustive & 14/14 & 136{,}852 & 1/7 & --- & 15/21 & --- & --- \\
\bottomrule
\end{tabular}
\end{table}

A separate L5 comparison with 42 paired runs per condition
shows the same direction of improvement: C2 finds a conflict
in 8 of 42 runs, compared with 5 of 42 for C1
(Table~\ref{tab:e1l5extension}). We report these results
separately from the 21-run L5 sample above.

\begin{table}[H]
\centering\footnotesize
\setlength{\tabcolsep}{6pt}
\caption{Additional L5 comparison with 42 paired runs per
condition, using the same 5,000-action budget and discovery
criterion as Table~\ref{tab:e1main}.}
\label{tab:e1l5extension}
\begin{tabular}{@{}lcc@{}}
\toprule
Arm & L5 found & Discovery rate \\
\midrule
C1 per-rule evidence & 5/42 & 11.9\% \\
C2 proof dependencies & 8/42 & 19.0\% \\
\bottomrule
\end{tabular}
\end{table}

\apppar{Additional tests of the verification procedure.}{app:lean-python}
We conduct mutation tests to examine what kinds of errors our
verification procedure can detect. In experiments, we deliberately introduce errors
into three models: the two sealed-level reference models (2{,}675
and 3{,}502 states) and the merged \code{lp85} model (869 states).
These errors affect proof certificates, model rules, consistency
between Python and Lean, or theorem statements. We compare a Python
certificate checker with our Lean\mbox{-}based verification procedure,
which also includes a separate audit of theorem statements.

Both methods reject all 12 invalid proof certificates, including
nine that would support conclusions contradicted by the game.
The Lean-based procedure additionally detects all three tested
Python--Lean mismatches and all three reversed bounds, which the
Python checker misses. Its separate statement audit also flags
all six cases where a theorem holds only because its premises
cannot be satisfied or its starting-state set is empty. Such
theorems may still compile in Lean: the kernel verifies that
conclusions follow from their premises, not whether those
premises are meaningful.

Neither method detects errors deliberately introduced into the
model's rules (M1--M6), or theorem statements referring to the
wrong goal or model version (S3, S7). These results illustrate
both the additional checks provided by Lean-based verification
and their limits: a valid proof does not guarantee that the
underlying model correctly describes the game.

\begin{table}[!htbp]
\centering\footnotesize
\setlength{\tabcolsep}{5pt}
\caption{Errors detected by the Python certificate checker
and the Lean-based verification procedure across three models.
Counts show detected cases out of tested cases. Vacuous
statements are flagged by a separate audit, not rejected
by the Lean kernel.}
\label{tab:leanpy}
\begin{tabularx}{\linewidth}{@{}Xcc@{}}
\toprule
Injected error & Python checker & Lean + audit \\
\midrule
Invalid proof certificates & 12/12 & 12/12 \\
Python--Lean model mismatches & 0/3 & 3/3 \\
Reversed bounds & 0/3 & 3/3 \\
Vacuous theorem statements & 0/6 & 6/6 flagged \\
Incorrect model rules & 0/18 & 0/18 \\
Wrong goal or model version & 0/6 & 0/6 \\
\bottomrule
\end{tabularx}
\end{table}

\section{Interview prompts}
\label{app:prompts}
This appendix reproduces verbatim the interview of Appendix~\ref{app:configs}. After each attempt the harness sends questions 1--3, the formalization questions (AWM runs only) and the verdict lines, in the same session as the play; the level is then reset with the retry message. After the last attempt the final interview replaces questions 1--3 and keeps the same verdict lines.

\lstdefinestyle{prompt}{
  basicstyle=\ttfamily\scriptsize,
  breaklines=true,
  breakatwhitespace=true,
  columns=fullflexible,
  keepspaces=true,
  frame=single,
  framesep=4pt,
  rulecolor=\color{black!25},
  backgroundcolor=\color{black!3},
  xleftmargin=4pt, xrightmargin=4pt,
  aboveskip=6pt, belowskip=8pt,
  showstringspaces=false,
  upquote=true,
  literate={–}{{--}}1 {—}{{---}}1 {’}{{'}}1 {‘}{{`}}1 {“}{{``}}1 {”}{{''}}1
}

\needspace{4\baselineskip}
\par\medskip\noindent\textbf{Questions 1--3, every run.}\par\nobreak
\begin{lstlisting}[style=prompt]
That attempt failed. The game will reset this level, and you may get another attempt.

Answer the following. Keep it to what it needs; do not pad.

1. What happened on the level you were on when the attempt ended, and what do you think made
   it fail? Explain in plain prose, as if to someone who did not watch you play.

2. Do you still see a way to beat this level?
   - If you do: say what it is, and what you would still need to check to be sure of it.
   - If you no longer do: say why not. State precisely what you are claiming, give the
     argument for it, and make clear what the argument rests on - which parts you established
     and which are assumptions you did not check. Someone else should be able to follow it and
     judge whether it holds. The form of the argument is yours to choose.
   - If you cannot tell either way, say so, and say what would settle it.

3. What will you do on the next attempt?
\end{lstlisting}

\needspace{4\baselineskip}
\par\medskip\noindent\textbf{Formalization questions, AWM runs.}\par\nobreak
\begin{lstlisting}[style=prompt]
4. Before you answer this one, bring your formal system up to date. It is the thing you are
   building across attempts; the proof is downstream of it. Work in this order, in your
   working directory:

   a. Objects. What things are on the board, and what slots does each thing carry? For every
      object you keep, you should be able to name a pair of positions that stop being
      distinguishable if you drop it.
   b. Rules. Write each rule as a function: given a position and an action, what does it
      change? Each rule needs a domain you can actually execute.
   c. Simplify. Drop slots no rule reads. Merge rules that always decide alike. After this,
      your model's verdict on every frame you have observed must be unchanged.
   d. Formalize. Turn the objects and rules into Lean *definitions* - not theorems. Each rule
      ends up in exactly one of three states: a definition, a consequence of definitions, or
      explicitly marked "not in the model".

   If you already have a formal system from an earlier attempt, **amend it. Do not start a new
   one.** What you learned this attempt is a change to specific objects, rules or definitions;
   make those changes and leave the rest alone. A rewrite throws away the part of your work
   that was already right, and it destroys the record of what this attempt actually taught you.

5. Now prove your answer to question 2 **on that formal system**: that the level can be won,
   or that it cannot.

   - The proof must compile, and `#print axioms` on your main theorem must come back clean -
     no `sorryAx`, no `native_decide`. Report that output verbatim.
   - Say what your theorem actually quantifies over and what it assumes. If the proof needs a
     premise you did not establish, write it as an explicit hypothesis rather than burying it
     in a definition.
   - Name the rules from step 4b that your theorem leans on. If your Lean definitions carry
     something your rules do not - an extra assumption, a tighter domain, a step treated as
     atomic - say so here. A proof that is clean about the wrong system is worth nothing, and
     a model that is *stricter* than the game is the dangerous direction: "impossible in my
     model" does not transfer to the game.
   - Put the file in `lean/` and tell me its filename.
   - If you cannot complete the proof, say exactly where it breaks and what is missing. An
     honest "it does not go through, here is the gap" is worth more than a theorem that proves
     something weaker than it looks.

6. If this attempt taught you nothing that changes your formal system - no rule revised, no
   definition touched - do not invent a change. Say so instead, on its own line, in exactly
   this form, and then give the reason:

       FORMAL-UNCHANGED: <why this attempt added nothing to the model - what you tried, what
       you already knew, and what would have had to happen for the model to change>

   "No change, and here is why" is a real result. A change made only to look busy is not.
\end{lstlisting}

\needspace{4\baselineskip}
\par\medskip\noindent\textbf{Verdict lines.}\par\nobreak
\par\smallskip\noindent\textit{AWM runs.}\par\nobreak
\begin{lstlisting}[style=prompt]
End your reply with exactly these two lines, on lines of their own, with nothing after them:

VERDICT: winnable | impossible | impossible_in_model_proved | impossible_in_model | undecided
CONFIDENCE: <an integer from 0 to 100>

The two "in_model" words are for when the level cannot be won under the rules you yourself
measured, and the step from "those rules" to "the real game" is an assumption you are stating
rather than something you established. Say which assumption, in your answer to question 2.
Use `impossible_in_model_proved` only if you are also handing over a Lean proof of it that
compiles and whose `#print axioms` is clean - name the file. Otherwise use `impossible_in_model`.
Your proof will be compiled and its axioms checked, so do not claim `_proved` on a file that
does not go through; `impossible_in_model` with an honest account of the gap is worth more.

Answer in writing only. Do not send any more game actions.
\end{lstlisting}
\par\smallskip\noindent\textit{Baseline runs.}\par\nobreak
\begin{lstlisting}[style=prompt]
End your reply with exactly these two lines, on lines of their own, with nothing after them:

VERDICT: winnable | impossible | undecided
CONFIDENCE: <an integer from 0 to 100>

Answer in writing only. Do not send any more game actions.
\end{lstlisting}

\needspace{4\baselineskip}
\par\medskip\noindent\textbf{Retry message.}\par\nobreak
\begin{lstlisting}[style=prompt]
You have another attempt on that level.

`POST /reset` restarts the level you were on; the levels you already beat stay beaten, and the
level's own counters start over. Play it again, as far as you can. When the attempt ends, stop
and report as before.
\end{lstlisting}

\needspace{4\baselineskip}
\par\medskip\noindent\textbf{Final interview, questions.}\par\nobreak
\begin{lstlisting}[style=prompt]
That attempt failed.

Answer the following in writing. Give each answer the length it needs; do not pad.

1. What happened on the level you were on, and what do you think made it fail?
   Explain in plain prose, as if to someone who did not watch you play.

2. Do you still see a way to beat this level?
   - If you do: give the line of play, and say what you checked it against.
   - If you no longer do: say why not. State precisely what you are claiming, give the
     argument for it, and make clear what the argument rests on - which parts you established
     and which are assumptions you did not check. Someone else should be able to follow it and
     judge whether it holds. The form of the argument is yours to choose.
   - If you cannot tell either way, say so, and say what would settle it.

3. Which rule did you newly work out across these attempts, and from which observation?

4. Which rule are you still unsure about, and unsure in what way?
\end{lstlisting}

\end{document}